\documentclass[conference]{IEEEtran}
\usepackage{times}

\usepackage[numbers]{natbib}
\usepackage{multicol}
\usepackage[bookmarks=true]{hyperref}

\usepackage{algpseudocode}

\usepackage{amsmath, amssymb}
\usepackage{xcolor}
\usepackage{graphicx}
\usepackage{comment}
\usepackage{subfig}
\usepackage[normalem]{ulem}
\usepackage{units}
\usepackage{algorithm}

\let\labelindent\relax
\usepackage{enumitem}

\newcommand{\argmin}[1]{\underset{#1}{\operatorname{arg}\operatorname{min}}\;}
\renewcommand{\min}[1]{\underset{#1}{\operatorname{min}}\;}
\newcommand{\minimize}[1]{\underset{#1}{\t{minimize}}\;}
\renewcommand{\max}[1]{\underset{#1}{\operatorname{max}}\;}
\renewcommand{\lim}[1]{\underset{#1}{\operatorname{lim}}\;}
\newcommand{\subjectto}{\text{subject to}}

\newcommand{\etal}{\textit{et al}. }

\makeatletter
\def\blfootnote{\xdef\@thefnmark{}\@footnotetext}
\makeatother

\usepackage{soul}
\newcommand{\new}[1]{#1}

\newcommand{\smplong}{Planning on Sequenced Manifolds}
\newcommand{\smp}{PSM$^*$}

\renewcommand{\v}{\boldsymbol}
\renewcommand{\t}[1]{{\textrm{#1}}}

\newcommand{\R}{\mathbb{R}}

\newcommand{\eq}{Equation~} 

\newcommand{\fig}{Figure~}
\newcommand{\alg}{Algorithm~}

\newcommand{\sect}{Section~}

\newcommand{\opt}{^\star}

\ifdefined\T
\renewcommand{\T}{^\top}
\else
\newcommand{\T}{^\top}
\fi

\newcommand{\mat}[3][.9]{
  \renewcommand{\arraystretch}{#1}{\scriptscriptstyle{\left(
    \hspace*{-1ex}\begin{array}{#2}#3\end{array}\hspace*{-1ex}
  \right)}}\renewcommand{\arraystretch}{1.2}
}

\newcommand{\specialcell}[2][c]{
	\begin{tabular}[#1]{@{}c@{}}#2\end{tabular}}

\newcommand{\inv}{^{-1}}

\newcommand{\comma}{~,}
\newcommand{\period}{~.}

\newcommand{\ga}{\alpha}
\newcommand{\gb}{\beta}

\renewcommand{\ge}{\epsilon}
\renewcommand{\gg}{\gamma}

\newcommand{\gl}{\lambda}

\newcommand{\gr}{\rho}

\newtheorem{lemma}{Lemma}

\newtheorem{theorem}{Theorem}
\newtheorem{definition}{Definition}

\begin{document}

\title{Sampling-Based Motion Planning on\\Sequenced Manifolds}

\author{\authorblockN{Peter Englert,
Isabel M.\ Rayas~Fernández,
Ragesh K.\ Ramachandran, 
Gaurav S.\ Sukhatme*}
Robotic Embedded Systems Laboratory \\University of Southern California, Los Angeles, CA}

\maketitle

\begin{abstract}
We address the problem of planning robot motions in constrained configuration spaces where the constraints change throughout the motion.
The problem is formulated as a \new{fixed} sequence of intersecting manifolds, which the robot needs to traverse in order to solve the task.
We specify a class of sequential motion planning problems that fulfill a particular property of the change in the free configuration space when transitioning between manifolds. 
For this problem class, the algorithm \smplong~(\smp) is developed which searches for optimal intersection points between manifolds by using RRT$^*$ in an inner loop with a novel steering strategy. We provide a theoretical analysis regarding \smp s probabilistic completeness and asymptotic optimality. Further, we evaluate its planning performance on multi-robot object transportation tasks. \\
Video: \textit{\url{https://www.youtube.com/watch?v=Q8kbILTRxfU}}\\
Code: \textit{\url{https://github.com/etpr/sequential-manifold-planning}}
\end{abstract}
\blfootnote{
\hspace{-1.7em}
* Gaurav Sukhatme holds concurrent appointments as a Professor at USC and as an Amazon Scholar. This paper describes work performed at USC and is not associated with Amazon.\\
Contact: \{penglert, rayas, rageshku, gaurav\}@usc.edu}

\IEEEpeerreviewmaketitle

\section{Introduction}
\label{sec:introduction}
Sampling-based motion planning (SBMP) considers the problem of finding a collision-free path from a start configuration to a goal configuration. Probabilistic roadmaps \cite{kavraki1994randomized} or rapidly exploring random trees \cite{lavalle1998rapidly} generate plans for such motions while providing theoretical guarantees regarding probabilistic completeness. Optimal algorithms like RRT$^*$ \cite{karaman2011sampling} additionally minimize a cost function and achieve asymptotic optimality. However, many tasks in robotics 
require the incorporation of additional objectives like subgoals or constraints.
SBMP methods are difficult to directly apply to such problems because they require splitting the overall task into multiple planning problems that fit into the SBMP structure.

We consider a 
formulation for sequential motion planning where a problem is represented as a fixed sequence of manifolds.
\begin{figure}[t]
	\centering
 	\includegraphics[width=.28\textwidth]{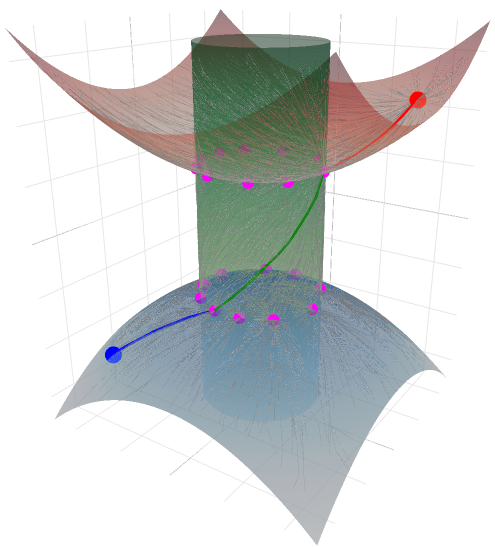}
	\caption{3D Point on Geometric Constraints -- The surfaces visualize the level sets of the three constraints. The task is to move from the start point (red dot) to the goal point (blue dot) while fulfilling the constraints. The line shows a solution found by \smp~that fulfills these constraints and the magenta points are the found intersection vertices.}
	\label{fig:hourglass}
	\vspace{-0.1in}
\end{figure}
We propose \smplong~(\smp), an algorithm that searches for an optimal path that starts at an initial configuration, traverses the manifold sequence, and converges when the final manifold is reached. 
Our solution is to grow a single tree over the manifold sequence.
This tree consists of multiple subtrees that originate at the intersections between pairs of manifolds. We propose a novel steering strategy that guides the robot towards these manifold intersections. After an intersection is reached, a new subtree is initialized with the found intersection points and their costs. 
We use dynamic programming over optimal intersection points which scales well to long horizon tasks since a new subtree is initialized for every manifold.
\new{The algorithm is applicable to a class of problems we call \emph{intersection point independent}, specified by a property which restricts how the free configuration space changes across subtasks (see \sect \ref{sec:problem_formulation}).}

A running example in this paper is the task of using a robot arm to transport a mug from one table to another while keeping the orientation of the mug upright, a task involving multiple phases and constraints described informally as follows.
First, the arm moves to pick up the mug. In this subtask, the arm can move freely in space and only needs to avoid collisions with obstacles. The second subtask is to grasp the mug. The third subtask is for the arm to transfer the mug with the constraint of always keeping the mug upright. The final subtask is to place and release the mug which requires the base of the mug to be near the table. 
In \sect \ref{sec:experiments}, we demonstrate the performance of \smp~on similar sequential kinematic planning problems that involve multiple robots and compare it to alternative planning strategies.

\new{The main contributions of this work are:
\begin{itemize}
    \item identification of the problem class of intersection point independent motion planning problems (\sect \ref{sec:problem_formulation}), which is a subclass of the more general multi-modal motion planning problem \cite{hauser2011randomized}, 
    \item  an algorithm, \smp, 
    that explores intersection points reachable from the start by growing a single tree over a given sequence of manifolds (\sect \ref{sec:method}),
    \item experimental evaluation of our method on various transportation tasks involving multiple robots and on a pouring task with a real Panda robot (\sect \ref{sec:experiments}), and comparisons to other methods,
    \item proofs of probabilistic completeness and asymptotic optimality (\sect \ref{ssec:comp-optimal}) of \smp.
\end{itemize}}

    \section{Related Work}
    \label{sec:related_work}
    \subsection{Sampling-Based Motion Planning}
    Sampling-based motion planning (SBMP) is a randomized approach to path planning that builds a tree or graph in the robot configuration space.
    A PRM (probabilistic roadmap) is a path planner that builds a graph in the free configuration space that can be used for multiple queries \cite{amato1996randomized, kavraki1994randomized, overmars1992random}. 
    Tree methods such as RRTs (rapidly-exploring random trees) are generally single-query, taking a specific start state from which a tree of feasible states is grown toward a specific goal state or region \cite{lavalle1998rapidly}. 
    Many extensions to RRT exist, such as bidirectional trees and goal biasing \cite{kuffner2000rrt, lavalle2006planning}. Optimal variants RRT$^*$ and PRM$^*$ find paths that minimize a cost function and guarantee asymptotic optimality by using a rewiring procedure on the edges in the graphs \cite{karaman2011sampling}. All these techniques consider the problem of planning without motion constraints; that is, they plan in the free configuration space. 
    
    We use a modified version of RRT$^*$ in the inner loop of \smp~that can handle goals defined in terms of equality constraints instead of a goal configuration. Restricting the problem class to intersection point independent problems (\sect \ref{ssec:ipip}) allows us to grow a single tree over a sequence of manifolds on which rewiring operations can still be performed. We show that \smp~inherits the probabilistic completeness and asymptotic optimality guarantees of RRT$^*$ (\sect \ref{ssec:comp-optimal}).

    \subsection{Constrained Sampling-Based Motion Planning}
    Many practical tasks require planning with constraints. 
    Constrained SBMP algorithms \cite{stilman2010global, berenson2011task, jaillet2013asymptotically, kim2016tangent, kingston2019ijrr, csucan2012motion, cortes2004sampling} extend SBMP to constrained configuration spaces (see \citet{kingston2018sampling} for a review). 
    These are of smaller dimension than the full configuration space and usually cannot be sampled directly, thus a large focus of these methods is how to generate samples and steer the robot while fulfilling the constraints.
    
    One family of approaches are projection-based strategies \cite{berenson2011task, stilman2010global, csucan2012motion, kaiser2012constellation} that first sample a configuration from the ambient configuration space and then project it using an iterative gradient descent strategy to a nearby configuration that satisfies the constraint. 
    Berenson \etal \cite{berenson2011task} proposed CBiRRT (constrained bidirectional RRT) that uses projections to find configurations on the constraints. The constraints are described by task space regions, which are a representation of pose constraints. Their method can also be used for multiple constraints over a single path. However, their approach requires each constraint's active domain to be defined prior to planning, or configurations are simply projected to the nearest manifold rather than respecting a sequential order. Our approach is perhaps closest to the CBiRRT algorithm. 
    A main difference is that our method assumes a different problem formulation where the task is given in terms of a fixed sequence of manifolds where the intersections between manifolds describe subgoals that the robot should reach. 
    This formulation allows us to define a more structured steering strategy that guides the robot towards the next subgoals. 
    Another difference is that CBiRRT does not search for optimal paths, while we employ RRT$^*$ to minimize path lengths. We compare our method \smp~to CBiRRT in \sect \ref{sec:experiments}.
    
    An alternative is to approximate the constraint surface by a set of local models and use this approximation throughout the planning problem for sampling or steering operations \cite{jaillet2013asymptotically, jaillet2013path, kim2016tangent,suh2011tangent, stilman2010global, bordalba2018randomized}.
    For example, AtlasRRT \cite{jaillet2013path} builds an approximation of the constraint consisting of local charts defined in the tangent space of the manifold. This representation is used to generate samples that are close to the constraint. 
    Similarly, Tangent Bundle RRT \cite{kim2016tangent, suh2011tangent} builds a bidirectional RRT by sampling a point on a tangent plane, extending this point to produce a new point, and if it exceeds a certain distance threshold from the center of the plane, projecting it on the manifold to create a new tangent plane. 
    
    Kingston \etal \cite{kingston2019ijrr} presented the implicit manifold configuration space (IMACS) framework that decouples two parts of a geometrically constrained motion planning problem: the motion planning algorithm and the method for constraint adherence. With this approach, IMACS acts as a representative layer between the configuration space and the planner. 
    Many of the previously mentioned techniques fit into this framework. They present examples with both projection-based and approximation-based methods for constraint adherence in combination with various motion planning algorithms. Here, we propose a method that builds on these constrained SBMP methods and extends them towards sequential tasks where the active constraints change during the motion. 
    
    \subsection{\new{Planning with Sequential Tasks}}
    One approach to plan sequential motions is task and motion planning (TAMP), which requires semantic reasoning on selecting and ordering actions to complete a higher-level task \cite{kaelbling2013integrated, dantam2016incremental, konidaris2018skills, toussaint2018differentiable, dantam2018task, barry2013hierarchical, cambon2009hybrid, kingston2020informing, pflueger2015multi}. 
    Broadly speaking, TAMP is more general and difficult to solve compared to SBMP due to scalability issues and a more complicated problem definition. 
    Though we do not address task planning in this work, it is an interesting future direction we plan to consider. 
    
    \citet{hauser2011randomized} proposed the multi-modal motion planning algorithm Random-MMP, which plans motions over multiple mode switches that describe changes in the planning domain (e.g., contacts). 
    Their planner builds a tree in a hybrid state using an SBMP that consists of the continuous robot configuration and a discrete mode, which changes based on the domain. 
    Other previous work \cite{vega2016asymptotically, kingston2020informing, hauser2011randomized, simeon2004manipulation, simeon2000visibility, mirabel2016hpp, schmitt2019modeling} addressed this problem in related ways. They generally assume sampling of the modes and their boundaries, and build graphs that connect these to each other.

    Our work is similar in that we also consider planning over multiple manifolds where changes in the configuration space occur due to picking or placing objects. 
    \new{However, our work differs in several key ways. First, TAMP approaches assume the task sequence is unknown in advance, and thus selecting the transition to the next constraint is part of the planning problem. Here we specify our problem formulation over a fixed and known sequence of constraints, focusing our algorithm on optimizing the transition points between constraints to find an optimal path.} 
    Additionally, we do not assume direct sampling of modes or switches is possible; rather, our algorithm is designed to find the mode-switching configurations during exploration toward manifold intersections. 
    \new{Importantly, because we build one search tree, these configurations are guaranteed to be connected to the start. This eliminates sampling of configurations that are disconnected from the viable solution space.}
    \new{Further, we differ by applying \smp~in the domain of intersection point independent problems, 
    which we specify for sequential planning problems based on the free space transformation at each manifold intersection (\sect \ref{ssec:ipip}).
    These problems do not include foliated manifolds and are efficiently solvable by growing a single tree.}
    
    Finally, trajectory optimization \cite{stryk1992direct, schulman2013finding, toussaint2017newton, ratliff2015understanding, pavone2019rss} is another approach to solve sequential motion planning problems where an optimization problem over a trajectory is defined that minimizes costs subject to constraints. 
    Our problem formulation is similar. \new{However, trajectory optimization often depends on having a good initial trajectory and can suffer from poor local minima. 
    The time points in the trajectory at which costs and constraints are active also need to be specified precisely, which can be challenging since it is to difficult to know in advance how long a specific part of a task takes.}
    Our approach only requires the sequential order of tasks and does not make any assumption of specific time points or durations of subtasks.
\section{Preliminaries}
An important idea in differential geometry (see \cite{Boothby:107707} for a rigorous treatment) is the concept of a \textit{manifold} -- a surface which can be well-approximated locally using an open set of a Euclidean space. In general, manifolds are represented using a collection of local regions called \textit{charts} and a continuous map associated with each chart such that the charts can be continuously deformed to an open subset of a Euclidean space. An alternate representation of manifolds, which is useful from a computational perspective, is to express them as zero level sets of functions defined on a Euclidean space. This is called an \textit{implicit representation} of the manifold. 
For example, a unit sphere embedded in $\R^3$ can be represented implicitly as \mbox{$\{x \in \R^3 ~|~ \|x\|-1 = 0\}$.} An implicit manifold is said to be \textit{smooth} if the implicit function associated with it is smooth. The set of all tangent vectors at a point on a manifold is a vector space called the \textit{tangent space} of the manifold at that point. The \textit{null space} of the Jacobian of the implicit function at a point is isomorphic to the tangent space of the corresponding manifold at that point. Since the tangent spaces of a manifold are vector spaces, we can equip them with an inner product structure which enables the computation of the length of curves traced on the manifold. A manifold endowed with an inner product structure is called a Riemannian manifold \cite{lee2006riemannian}. Here we only consider smooth Riemannian manifolds. 
\section{Problem Formulation and Application Domain}
\label{sec:problem_formulation}
We consider kinematic motion planning problems in a configuration space $C\subseteq \R^k$. A configuration $q\in C$ describes the state of one or multiple robots with $k$ degrees of freedom in total. 
We represent a manifold $M$ as an equality constraint $h_{M}(q)=0$ where $h_{M}(q) : \R^{k} \to \R^{l}$
and $l$ is the intrinsic dimensionality of the constraint ($l \leq k$). The set of robot configurations that are on a manifold $M$ is given by $C_{M} = \{q\in C ~|~ h_{M}(q) = 0\}\period$
We define a projection operator $q_\t{proj} =  \textit{Project}(q, M)$
that takes a robot configuration $q\in C$ and a manifold $M$ as inputs and maps $q$ to a nearby configuration on the manifold $q_\t{proj}\in C_M$. 
We use an iterative optimization method, similar to \cite{kingston2019ijrr, berenson2011task, stilman2007task}, that iterates 
$$q_{n+1} = q_n - J_{M}(q_n)^+ h_{M}(q_n)$$
until a fixed point on the manifold is reached, which is checked by the condition $||h_M(q_\t{proj})|| \leq \ge$. The matrix $J_{M}(q)^+$ is the pseudo-inverse of the constraint Jacobian $J_{M}(q) = \frac{\partial }{\partial q} h_{M} (q)\period$

We are interested in solving tasks that are defined by a sequence of $n+1$ such manifolds 
$\mathcal{M} = \{M_1, M_2, \dots, M_{n+1}\}$
\new{where the number and order of manifolds is given,} and an initial configuration $q_\t{start}\in C_{M_1}$ that is on the first manifold.
The goal is to find a path from $q_\t{start}$ that traverses the manifold sequence $\mathcal{M}$ and reaches a configuration on the goal manifold $M_{n+1}$.
A path on the $i$-th manifold is defined as \mbox{$\tau_i : [0, 1] \to C_{M_i}$}
and $J(\tau_i)$ is a cost function of a path
\mbox{${J} : \mathcal{T} \to \R_{\geq 0}$}
where $\mathcal{T}$ is the set of all non-trivial paths. 

The free configuration space $C_\t{free}$ is the subset of the configuration space that does not result in the robot colliding with an obstacle. In the following problem formulation, we consider the scenario where the free configuration space changes during the task (e.g., when picking or placing objects). 
We define an operator $\Upsilon$ on the free configuration space $C_\t{free}$, which describes how $C_\t{free}$ changes as manifold intersections are traversed. 
$\Upsilon$ takes as input the \new{endpoint of a} path $\tau_{i-1}(1)$ on the previous manifold $M_{i-1}$ and its associated $C_\t{free}$, which we denote as $C_{\t{free}, i-1}$. $\Upsilon$ outputs an updated free configuration space 
$C_{\t{free}, i} = \Upsilon(C_{\t{free}, i-1}, \tau_{i-1}(1))$
that accounts for the geometric changes due to transitioning to a new manifold. 
\new{Here, we assume these changes only occur at the transition between two manifolds.}
\new{We define $\Upsilon$ in this way rather than with the configuration spaces of the robot and the objects in the environment because our formulation does not have a notion of ``objects"; any object in the world is either an obstacle, or it becomes part the robot system itself.}

We assume access to a collision check routine
$\textit{CollisionFree}(q_a, q_b, C_{\t{free}, i}) \to \{0,1\}$ that returns $1$ if the path between two configurations on the manifold $q_a, q_b \in C_{M_i}$ is collision-free, $0$ otherwise. Note that in this work, we use the straight-line path in ambient space between two nearby points for checking collisions.

Returning to our illustrative example, we can now describe one of its constraints more precisely. A grasp constraint can be described with \mbox{$h_M(q) = x_g - f_{\t{pos}, e}(q)$} where $f_{\t{pos}, e}$ is the forward kinematics function of the robot end effector point $e$ and $x_g$ is the grasp location on the mug. 
This constraint can be fulfilled for multiple robot configurations $q$, which correspond to different hand orientations that will affect the free configuration space for the subsequent tasks. For example, a mug grasped from the side will have a different free space during the transport phase than a mug grasped from the top.

\subsection{Problem Formulation}
We formulate an optimization problem over a set of paths $\v\tau = (\tau_1, \dots, \tau_{n})$ that minimizes the sum of path costs under the constraints of traversing $\mathcal{M}$ and of being collision-free. The \emph{planning on sequenced manifolds problem} is
\begin{align} 
	\begin{alignedat}{2} 
	\label{eq:smp_problem}
	\v\tau\opt =~ &\argmin{\v\tau} \sum_{i=1}^n J(\tau_i) &&\\
	\text{s.t.}\quad &\tau_1 (0) = q_\t{start} &\\
	&\tau_i(1) = \tau_{i+1}(0)  &&\forall_{i=1,\dots,n-1} \\ 
	&C_{\t{free}, i+1} = \Upsilon(C_{\t{free}, i}, \tau_{i}(1))~~~~ &&\forall_{i=1,\dots,n}\\
	&\tau_i(s) \in C_{M_i} \cap C_{\t{free}, i} &&\forall_{i=1,\dots,n}~ \forall_{s \in [0, 1]}\\
	&\tau_{n}(1) \in C_{M_{n+1}} \cap C_{\t{free}, n+1} &&
	\end{alignedat} 
\end{align}
The second constraint ensures continuity such that the endpoint of a path $\tau_i(1)$ is the start point of the next path $\tau_{i+1}(0)$. The third constraint captures the change in the collision-free space defined by the operator $\Upsilon$. The last two constraints ensure that the path is collision-free and on the corresponding manifolds. The endpoint of $\tau_{n}$ must be on the goal manifold $M_{n+1}$, which denotes the successful completion of the task.
\new{Note that this problem can be seen as a special case of the multi-modal motion planning problem presented in \cite{hauser2011randomized} where the manifolds are provided in their sequential order.} The problem defined in \eq \eqref{eq:smp_problem} is difficult to solve directly for multiple reasons. One difficulty lies in computing gradients of the cost and constraints through manipulation operations like picking and placing objects. Another difficulty is the specification of good initial guesses for the individual paths.

\new{There are several advantages to formulating the problem with manifolds.} 
One is that it is not necessary to choose a specific target configuration in $C$ and thus a wider range of goals 
can be described in the form of manifolds. 
Another is that it is possible to describe complex sequential tasks in a single planning problem, not requiring the specification of subgoal configurations. 
The algorithm proposed in \sect \ref{sec:method} searches for a path that solves this optimization problem for the problem class described in the next section.

\begin{figure}[t]
    \centering
    \includegraphics[width=1\linewidth]{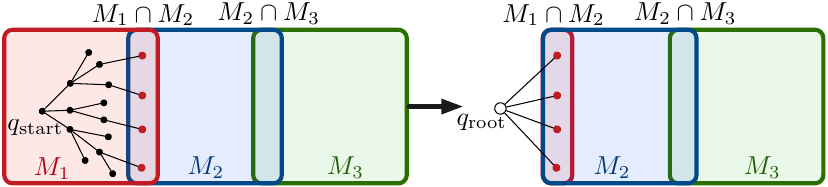}
    \caption{Initialization of a new subtree (steps \ref{alg:newtreeinit_start}--\ref{alg:newtreeinit_end} of \alg \ref{alg:smp}). The reached goal nodes at the intersection $M_1\cap M_2$ are used to initialize the next tree where $q_\t{root}$ is a synthetic root node that maintains the tree structure.}
    \label{fig:tree_init}
    \vspace{-0.25in}
\end{figure}

\subsection{Intersection Point Independent Problems}
\label{ssec:ipip}
We now use 
\eqref{eq:smp_problem} to describe robotic manipulation planning problems in which the manifolds describe subtasks for the robot to complete (e.g., picking up objects). The solution is an end-to-end path across multiple task constraints.

For a certain class of sequential manifold planning problems, the following property holds: For each manifold intersection \mbox{${M}_{i} \cap {M}_{i+1}$}, the free space output by $\Upsilon$ is set-equivalent for every possible path $\tau_{i}$. In other words, the precise action taken to move the configuration from one constraint to the next does not affect the feasible planning space for the subsequent subtask. When this property holds for all intersections, we call the problem \emph{intersection point independent}. The condition for this class of problems is
\begin{align}
	\begin{alignedat}{2} 
		\label{eq:ip_independence}
		&\forall i \in [0, n] ~\forall \tau_i, \tau_i' \in \mathcal{T} : \tau(1), \tau'(1) \in C_{M_i} \cap C_{M_{i+1}} \\
		&\Rightarrow\Upsilon(C_{\t{free}, i}, \tau(1)) \equiv \Upsilon(C_{\t{free}, i}, \tau'(1))
	\end{alignedat}
\end{align}
where $\equiv$ denotes set-equivalence. In this work, we focus on the intersection point independent class of motion planning problems, which encompass a wide range of common problems.
For grasping constraints, a notion of object symmetry about the grasp locations results in intersection point independent problems. If the object to be grasped is a cylindrical can, for example, allowing grasps to occur at any point around the circumference but at a fixed height would be intersection point independent. Any two grasps with the same relative orientation of the gripper result in the same free configuration space of the system (robot + can).
However, suppose the grasps can occur at any height on the can. Now, a grasp near the top of the can and a grasp in the middle of the can result in different free configuration spaces, and thus this would be an intersection point dependent problem.

Focusing on intersection point independent problems allows us to define an efficient algorithm that grows a single tree over a sequence of manifolds. The more general intersection point dependent problem covers a wider class of problems. However, they are more difficult to solve because they require handling foliated manifolds \cite{kim2014randomized} (e.g., every grasp leads to a different manifold). 
Our illustrative example is one such problem, since the grasp on the handle and the grasp from the top of the mug result in different free spaces. We plan to address this more complex problem class in future work and provide some insights in \sect \ref{sec:conclusion} how the proposed algorithm could be used to tackle it.
\section{\smplong}
\label{sec:method}
We now present the algorithm \smp~that solves the problem formulated in \eq \eqref{eq:smp_problem} \new{for tasks that fulfill the intersection point independent property defined in \eq \eqref{eq:ip_independence}}.
The algorithm searches for an optimal solution to the constrained optimization problem, which is a sequence of paths $\v\tau = (\tau_1, \dots, \tau_{n})$ where each $\tau_i$ is a collision-free path on the corresponding manifold $M_i$. 

The steps of \smp~are outlined in \alg \ref{alg:smp}.
The input to the algorithm is a sequence of manifolds $\mathcal{M}$, an initial configuration \mbox{$q_\t{start}\in M_1$} on the first manifold, 
\new{and the hyperparameters of the algorithm $(\ga, \gb, \gr, \ge, r, m)$ (see Table \ref{tab:exp_parameter}).
$V, E$ are the vertices and edges, respectively, in the tree that the algorithm builds, while $C$ is the robot configuration space and $C_{\t{free}, 1}$ is the initial free configuration space.}

The overall problem is divided into $n$ subproblems, \new{which correspond to the number of manifold intersections in $\mathcal{M}$.} Each subproblem \new{can be broadly thought of as growing a tree between two consecutive manifold intersections.} We give a high-level summary here.
In the inner loop (lines \ref{alg:innerloop_start}--\ref{alg:innerloop_end}; see \sect \ref{ssec:phase12}), a tree is iteratively grown from a set of initial nodes toward the intersection with the next manifold.
First, a new constraint-adhering configuration $q_\t{new}$ is found by sampling a point, steering, and projecting it (steps \ref{alg:steer_start}--\ref{alg:steer_project_step}).
Second, in steps \ref{alg:check_proj}--\ref{alg:extend_end}, the point is checked for validity before being added to the tree as well as the set of intersection points $V_\t{goal}$ (if applicable).
\new{After the inner loop completes}, the algorithm initializes the next subproblem with the found intersection points in steps \ref{alg:newtreeinit_start}--\ref{alg:next_freespace} (see \sect \ref{ssec:phase3}). It returns the optimal path that traverses all manifolds in $\mathcal{M}$.

\subsection{Inner Loop: Growing a Tree to the Next Manifold}
\label{ssec:phase12}
The first phase of the inner loop focuses on producing a new candidate configuration $q_\t{new}$ which lies on the constraint manifold $M_i$ and adding it to the tree. 
The \smp \_STEER routine in \alg \ref{alg:steer_project} computes $q_\t{new}$.
Instead of targeting a single goal configuration as in general SBMP, we propose two novel steering strategies that steer toward the intersection between manifolds $M_i\cap M_{i+1}$. 
In steps \ref{alg:steer_start}--\ref{alg:steer_end} of \alg \ref{alg:smp}, a new target point $q_\t{rand}$ in the configuration space is sampled and its nearest neighbor $q_\t{near}$ in the tree is computed. Next, in \alg \ref{alg:steer_project} one of the following two steering strategies is selected to find a direction $d$ in which to extend the tree:

\subsubsection{\new{$\t{SteerPoint}(q_\t{near}, q_\t{rand}, M_i)$}}
\label{sec:steer_point}
In this extension step, the robot is at $q_\t{near}\in C_{M_i}$ and should step towards the target configuration $q_\t{rand}\in C$ while staying on the manifold $M_i$. We formulate this problem as a constrained optimization problem of finding a curve $\gg ~:~ [0,1]\to C$ that
\begin{align}
    \begin{alignedat}{2}
    	&\minimize{\gg} &&||\gg(1) - q_\t{rand}||^2\\
		&\subjectto~~&&\gg(0) = q_\t{near},\quad \int_{0}^1 ||\dot{\gg}(t)|| ~\mathrm{d}s \leq \ga\\
		& & &h_{M_i}(\gg(s)) = 0 \quad\qquad \forall~ s\in [0,1]
	\label{eq:steer_point}
	\end{alignedat}
\end{align}
This problem is hard to solve due to the nonlinear constraints. Since the steering operations are called many times in the inner loop of the algorithm, we choose a simple curve representation and only compute an approximate solution to this problem. We parameterize the curve as a straight line $\gg(s) = q_\t{near} + s\ga \frac{d}{||d||}$ with length $\ga$
\new{where the direction $d$ is chosen as orthogonal projection of $q_\t{rand}-q_\t{near}$ onto the tangent space of the manifold at $q_\t{near}$.}
We apply a first-order Taylor expansion of the manifold constraint $h_{M_i}(q_\t{near} + d) \approx h(q_\t{near}) + J_{M_i}(q_\t{near}) d$, which reduces the problem to
\begin{align}
	\begin{alignedat}{3}
	&\minimize{d} &&\frac{1}{2} ||d  - (q_\t{rand}-q_\t{near})||^2\\
	&\subjectto~~ &&J_{M_i}(q_\t{near}) d = 0
	\end{alignedat}
\end{align}
The optimal solution of this problem is
\begin{align}
	\begin{alignedat}{2}
	d &=(I -  J_{M_i}\T (J_{M_i} J_{M_i}\T)\inv J_{M_i}) (q_\t{rand}-q_\t{near})\\
	&= V_{\bot} V_{\bot}\T (q_\t{rand}-q_\t{near})
	\label{eq:steer_point_sol}
	\end{alignedat}
\end{align}
where $V_{\bot}$ contains the singular vectors that span the right nullspace of $J_{M_i}$ \cite{ratliff2014multivariate}. We normalize $d$ later in the algorithm, and thus do not include the constraint $||d||\leq \alpha$ in the reduced optimization problem. The new configuration $q_\t{near} + \alpha \frac{d}{||d||}$ will be on the tangent space of the manifold at configuration $q_\t{near}$, so that only few projection steps will be necessary before it can be added to the tree. \new{Note that projection may fail to converge, and if this occurs the sample is discarded.}

\begin{algorithm}[t]
\begin{algorithmic}[1]
	\algrenewcommand\algorithmicindent{.8em}
	\State $V_1 = \{q_{\t{start}}\}$; $E_1 = \emptyset$; $n=\text{len}(\mathcal{M})-1$
	\For{$i=1$ to $n$}
		\State $V_\t{goal} = \emptyset$
		\For{$k = 1$ to $m$} \label{alg:innerloop_start}
			\State $q_\t{rand} \leftarrow \t{Sample}(C)$  \label{alg:steer_start}
			\State $q_\t{near} \leftarrow \t{Nearest}(V_i, q_\t{rand})$
			\State{$q_\t{new} \leftarrow \t{\smp\_STEER} (\alpha, \beta, r, q_\t{near}, q_\t{rand}, M_i, M_{i+1}$)} \label{alg:steer_project_step}
	\If{\new{$||h_{M_{i}} (q_\t{new})|| > \ge$}} \label{alg:check_proj}
	    \State continue
	\EndIf
			\If{RRT$^*$\_EXTEND$(V_i, E_i, q_\text{near}, q_\text{new}, C_{\text{free},i})$} \label{alg:extend_start}
				\If{$||h_{M_{i+1}} (q_\t{new})|| < \ge$ \label{alg:next_manifold} \textbf{and}\\ $\hspace{5.5em} ||\t{Nearest}(V_\t{goal}, q_\t{new}) - q_\t{new}|| \geq \gr$} \label{alg:avoid_duplicates}
					\State $V_\t{goal} \leftarrow V_\t{goal} \cup \{q_\t{new}\}$
				\EndIf
			\EndIf \label{alg:extend_end}
        \EndFor \label{alg:innerloop_end}
		\State{// initialize next tree with the intersection nodes}
		\State $q_\t{root}=\t{null}, V_{i+1} = \{q_\t{root}\}; E_{i+1}=\emptyset$ \label{alg:newtreeinit_start}
		\For{$q \in V_\t{goal}$}
			\State $V_{i+1} \leftarrow V_{i+1} \cup \{q \}; E_{i+1} \leftarrow E_{i+1} \cup \{(q_\t{root}, q )\}$
		\EndFor\label{alg:newtreeinit_end}
		\State{\new{$C_{\t{free},i+1} \leftarrow \Upsilon(C_{\t{free},i}, V_\t{goal}[0])$}} \label{alg:next_freespace}
	\EndFor
	\State \textbf{return} $\mathrm{OptimalPath}(V_{1:n}, E_{1:n}, q_\t{start}, M_{n+1})$
\end{algorithmic}
\caption{\smp~$(\mathcal{M}, q_{\t{start}}, \ga, \gb, \ge, \rho, r, m)$}
\label{alg:smp}
\end{algorithm}

\begin{algorithm}[t]
\begin{algorithmic}[1]
	\algrenewcommand\algorithmicindent{1.0em}
	\If{\new{$\t{Sample}(\mathcal{U}(0,1)) < \gb$}}
		\State $d \leftarrow \t{SteerConstraint}(q_\t{near}, M_i, M_{i+1})$ 
	\Else
		\State $d \leftarrow \t{SteerPoint}(q_\t{near}, q_\t{rand}, M_i)$
	\EndIf
	\State $q_\t{new} \leftarrow q_\t{near} + \alpha \frac{d}{||d||}$ \label{alg:steer_end}
	\If{$||h_{M_{i+1}} (q_\t{new})|| <$ \new{$ \t{Sample}(\mathcal{U}(0,r))$}} \label{alg:project_start}
		\State $q_\t{new} \leftarrow \t{Project}(q_\t{new}, M_{i} \cap M_{i+1})$
	\Else
		\State $q_\t{new} \leftarrow \t{Project}(q_\t{new}, M_{i})$
	\EndIf \label{alg:project_end}
	\State \textbf{return} $q_\t{new}$
\end{algorithmic}
\caption{$\t{\smp\_STEER} (\alpha, \beta, r, q_{\t{near}}, q_{\t{rand}}, M_i, M_{i+1})$}
\label{alg:steer_project}
\end{algorithm}

\subsubsection{\new{$\t{SteerConstraint}(q_\t{near}, M_i, M_{i+1})$}}
\label{sec:steer_constraint}
This steering step extends the tree from $q_\t{near}\in C_{M_i}$ towards the intersection of the current and next manifold $M_i\cap M_{i+1}$, which can be expressed as the optimization problem
\begin{align}
\begin{alignedat}{2}
	&\minimize{\gg} &&||h_{M_{i+1}}(\gg(1))||^2 \\
	&\subjectto\quad &&\gg(0) = q_\t{near},\quad \int_{0}^1 ||\dot{\gg}(t)|| ~\mathrm{d}s \leq \ga\\ 
	&&&h_{M_i}(\gg(s)) = 0 \quad\qquad \forall~ s\in [0,1]
	\end{alignedat}
\end{align}
The difference from problem \eqref{eq:steer_point} is that the loss is now specified in terms of the distance to the next manifold $h_{M_{i+1}}(\gg(1))$. This cost pulls the robot towards the manifold intersection. Again, we approximate the curve with a line $\gamma(s)$ and apply a first-order Taylor expansion to the nonlinear terms, which results in the simplified problem
\begin{align}
\begin{alignedat}{2}
	&\minimize{d} &&\frac{1}{2}||h_{M_{i+1}}(q_\t{near}) + J_{M_{i+1}}(q_\t{near})d||^2\\
	&\subjectto~~&&J_{M_i}(q_\t{near}) d = 0\period
	\end{alignedat}
\end{align}
A solution $d$ can be obtained by solving the linear system
\begin{align}
    \label{eq:steer_constraint_sol}
    \mat{c c}{J_{M_{i+1}}\T J_{M_{i+1}} & J_{M_{i+1}}\T \\ J_{M_{i}} & 0} \mat{c}{d\\ \gl} = \mat{c}{-J_{M_{i+1}}\T h_{M_{i+1}}\\ 0}
\end{align}
where $\gl$ are the Lagrange variables. The solution is in the same direction as the steepest descent direction of the loss projected onto the tangent space of $M_i$.
Similar to the goal bias in RRT, a parameter $\gb\in[0,1]$ specifies the probability of selecting the SteerConstraint step rather than SteerPoint.

After the steering strategy is determined and $d$ is computed, $q_\t{new}$ is projected either on $M_i$ or on the intersection manifold $M_i \cap M_{i+1}$ depending on the distance to the intersection of the manifolds measured by $||h_{M_{i+1}} (q_\t{new})||$ (step \ref{alg:project_start} of \alg \ref{alg:steer_project}).
The threshold for this condition is sampled uniformly between $0$ and $r$, where the parameter $r\in \R_{>0}$ describes the closeness required by a point around $M_{i+1}$ to be projected onto $M_i \cap M_{i+1}$. 
This randomization is necessary in order to achieve probabilistic completeness that is discussed in \sect \ref{ssec:comp-optimal}, which also gives a formal definition of $r$. 

At this point, $q_\t{new}$ has been produced and the algorithm attempts to add it to the tree using the extend routine from RRT${}^*$ \cite{karaman2011sampling}, which we outline in Appendix \ref{sec:rrt_extend}.

\new{The final step in each inner iteration of \alg \ref{alg:smp} is to check two necessary conditions to determine if the point should also be added to the set of intersection nodes $V_\t{goal}$: 
\mbox{1) The} point $q_\t{new}$ has to be on the next manifold $h_{M_{i+1}}$ (step \ref{alg:next_manifold}); 
2) $V_\t{goal}$ does not already contain the point $q_\t{new}$ or a point in its vicinity (step \ref{alg:avoid_duplicates}).
For the second condition, we introduce the parameter $\gr \in \R_{\geq 0}$ that is the minimum distance between two intersection points. As we will discuss in the theoretical analysis, $\gr$ must be $0$ in order to achieve probabilistic completeness and asymptotic optimality. In practice, we usually select $\gr > 0$, which results in a better performance since it avoids having a large number of nearby and duplicate intersection points.}

\subsection{Outer Loop: Initializing the Next Subproblem}
\label{ssec:phase3}
After the inner loop of \smp~completes, a new tree is initialized in steps \ref{alg:newtreeinit_start}--\ref{alg:newtreeinit_end} with all the intersection nodes in $V_\t{goal}$ and their costs so far. To keep the tree structure, we add a synthetic root node $q_\t{root}$ as parent node for all intersection nodes (\fig \ref{fig:tree_init}). In line \ref{alg:next_freespace}, the free configuration space is updated based on the reached intersection node. Convergence of the algorithm can be defined in various ways. We typically set an upper limit to the number of nodes or provide a time limit for the inner loop. After reaching the convergence criteria, the algorithm returns the path with the lowest cost that reached the goal manifold $M_{n+1}$.

\begin{table}[t]
\centering
\renewcommand{\tabcolsep}{4pt}
    \begin{tabular}{|c|c|c|c|}
        \hline
        \textbf{Parameter name} & \textbf{Symbol} & \specialcell{\textbf{3D point}\\ (Sec. \ref{ssec:geom_exp})} & \specialcell{\textbf{Robot transport}\\ (Sec. \ref{ssec:transport_exp})}\\
        \hline
        max step size & $\alpha$ & $1.0$ & $1.0$\\
        goal bias probability & $\beta$ & $0.1$ & $0.3$\\
        constraint threshold & $\epsilon$ & $0.01$ & \unit[1e-5]\\
        duplicate threshold & $\rho$ & $0.1$ & $0.5$\\
        projection distance & $r$ & $1.5$ & $0.5$\\
        number of samples & $m$ & $1200$ & $2000$\\
        \hline
    \end{tabular}
    \caption{Parameters of the \textit{3D Point on Geometric Constraints} and \textit{Multi-Robot Object Transport} experiment.}
    \label{tab:exp_parameter}
    \vspace{-0.12in}
\end{table}

\subsection{Completeness and Optimality}
\label{ssec:comp-optimal}
Under the assumption that $\rho$ is 0, \smp~is provably probabilistically complete and asymptotically optimal. Here, we outline our approach to proving these properties of \smp~(detailed proofs can be found in the appendix). Probabilistic completeness is proved in two steps. First we prove probabilistic completeness of \smp~on a single manifold. We claim subsequently that almost surely the tree grown on a manifold can be expanded onto the next manifold as the number of samples tends to infinity. The calculations related to the first step are based on \cite[Theorem 1]{Kleinbort2019}. Theorem 1 in Appendix \ref{sec:Probabilistic completeness} proves the subsequent claim. 

Proving the asymptotic optimality of \smp~relies on three lemmas (Lemma 2, 3 and 4 in Appendix \ref{sec:Asymptotic optimality}). Lemma 2 shows that for any weak clearance path \cite{karaman2011sampling} there exist a sequence of paths with strong clearance \cite{karaman2011sampling} that converges to a path with weak clearance. Our analysis assumes that the optimal path has weak clearance. Hence, there exist a sequence of strong clearance paths that converges to the optimal path.  Using Lemma 3 and Lemma 4, we prove that with an appropriate choice of the parameter $\gr$, the \smp~tree grown on the sequenced manifolds contains the paths which are arbitrarily close to any strong clearance path in the sequence of paths that converges to the optimal path. Finally using continuity of the cost function, we prove that \smp~is asymptotically optimal.

\section{Evaluation}
\label{sec:experiments}
In the following experiments, we solve kinematic motion planning problems where the cost function measures path length. We compare the following methods with each other:

\begin{itemize}[leftmargin=*]
    \item \textbf{\smp}: The method proposed here -- \alg \ref{alg:smp}.
    \item \textbf{\smp~(Greedy)}: Algorithm \ref{alg:smp} with the modification that only the node with the lowest cost in $V_\t{goal}$ is selected to initialize the next tree (steps \ref{alg:newtreeinit_start} -- \ref{alg:newtreeinit_end} of \alg \ref{alg:smp}).
    \item \textbf{\smp~(Single Tree)}: This method grows a single tree over the manifold sequence without splitting it into individual subtrees. The algorithm uses $\rho=0$. It is explained in detail in Algorithm 3 in Appendix \ref{sec:smp_single_tree}.
    \item {\textbf{RRT$^*$+IK}}: This method consists of a two-step procedure that is applied to every manifold in the sequence. First, a goal point on the manifold intersection is generated via inverse kinematics by randomly sampling a point in $C$ and projecting it onto the manifold intersection. Next, RRT$^*$ is applied to compute a path on the current manifold towards this node \cite{kingston2019ijrr}. This procedure is repeated until a point on the goal manifold is reached.
    \item {\textbf{Random MMP}}: The Randomized Multi-Modal Motion Planning algorithm \cite{hauser2011randomized} was originally developed for the more general class of motion planning problems where the sequence of manifolds is unknown. Here, we apply it to the simpler problem where the sequence is given in advance and use RRT$^*$ to connect points between mode transitions.
    \item {\textbf{CBiRRT}: The Constrained Bidirectional Rapidly-Exploring Random Tree algorithm \cite{berenson2009manipulation, berenson2011task} grows two trees towards each other with the RRT Connect strategy \cite{kuffner2000rrt}. In each steering step, the sampled configuration is projected onto the manifold with the lowest level set function value. Since CBiRRT does not optimize an objective function, we apply a short cutting algorithm as a post-processing step.}
\end{itemize}
We compare these methods on the following criteria:
\begin{itemize}[leftmargin=*]
    \item \textbf{Path length} -- The length of the found path in $C$ space.
    \item \textbf{Success rate} -- The number of times a collision-free path to the goal manifold is found for different random seeds.
    \item \textbf{Computation time} -- Time taken to compute a path. All experiments are run on a 2.2 GHz Quad-Core Intel Core i7.
\end{itemize}
\new{The parameter values $\alpha, \beta, \epsilon, \rho, r, m$ for the individual experiments were chosen experimentally and are summarized in Table \ref{tab:exp_parameter}.}
\subsection{3D Point on Geometric Constraints}
\label{ssec:geom_exp}
\new{We demonstrate the planning performance and properties of the individual methods on a simple point that can move in a 3D space.} The point needs to traverse three constraints defined by geometric primitives. The configuration space is limited to $[-6, 6]$ in all three dimensions.
The initial state is \mbox{$q_\t{start}=(3.5, 3.5, 4.45)$} and the sequence of manifolds is
\begin{enumerate}
	\item Paraboloid: $h_{M_1}(q) = 0.1 q_1^2 + 0.1 q_2^2 + 2 - q_3$
	\item Cylinder: $h_{M_2}(q) = 0.25 q_1^2 + 0.25 q_2^2 - 1.0$
	\item Paraboloid: $h_{M_3}(q) = -0.1 q_1^2 - 0.1 q_2^2 - 2 - q_3$
	\item Goal point:  $h_{M_4}(q) = q - q_\t{goal}$ with the goal configuration \mbox{$q_\t{goal}=(-3.5, -3.5, -4.45)$}.
\end{enumerate}
We evaluate the algorithms on two variants of this problem: an obstacle-free variant (\fig \ref{fig:hourglass}), and a variant that contains four box obstacles placed at the intersections between the manifolds (\fig \ref{fig:hourglass_paths}). In \fig \ref{fig:hourglass}, $q_\t{start}$ is drawn as red point and $q_\t{goal}$ as blue point. The intersection nodes $V_\t{goal}$ are shown as magenta points and a solution path from \smp~is visualized as a line. 

\begin{figure}[t]
	\centering
	\includegraphics[width=.45\textwidth]{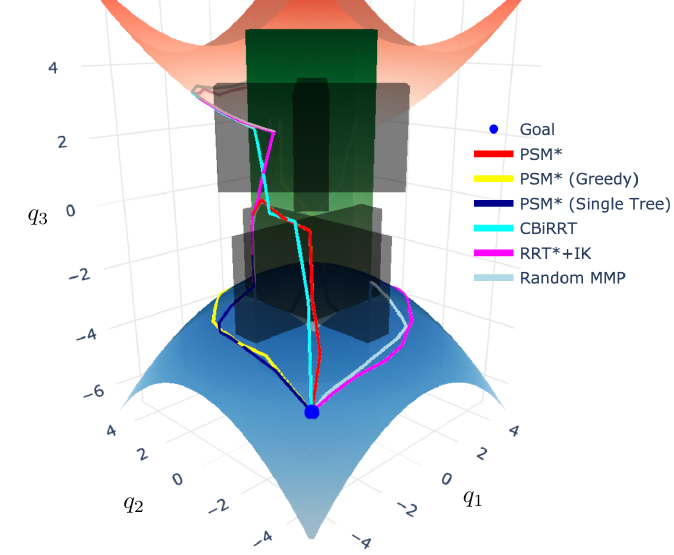}
	\caption{Samples of found paths on the \textit{3D point on geometric constraints w/ obstacles} problem (\sect \ref{ssec:geom_exp}).}
	\label{fig:hourglass_paths}
	\vspace{-0.25in}
\end{figure}

\new{
The results on success rate, path length and computation time are given in Table \ref{tab:comparison_results}. All the methods are consistently able to find a path for all $10$ random seeds. \smp~and \smp~(Single Tree) consistently achieve the lowest cost in both scenarios. \smp~(Single Tree) has a higher computation time, because it keeps all nodes in a single tree whereas \smp~splits them into individual trees per manifold. CBiRRT does not optimize an objective function and immediately converges when it finds a feasible path, which achieves the overall lowest computation time. However, the mean cost and standard deviations are higher compared to \smp. RRT${}^*$+IK and Random MMP only optimize over the individual paths, but do not optimize the intersection point selection, which results in higher costs and standard deviations. 
\smp~(Greedy) is better in terms of computation time since it only takes one intersection point as initial point for the next tree, but achieves an overall lower performance compared to \smp. 
The results also show that only \smp~and \smp~(Single Tree) are consistently able to find the intersection regions between the obstacles in which the optimal path lies. The other methods mainly select the intersection regions that are discovered first in the exploration. \fig \ref{fig:hourglass_paths} shows a set of found paths on the variant with obstacles.}

\new{In \fig \ref{fig:hourglass_rho_cost}, the path costs $J(\tau)$ of \smp~are compared for various values of $\gr$. This parameter specifies the minimum distance between two intersection points, which influences the number of intersection points created during planning (step \ref{alg:avoid_duplicates} in \alg \ref{alg:smp}). As a reference, we visualize the path costs of \smp~(Greedy) and \smp~(Single Tree). The results show that the computed paths of \smp~with a small $\gr$ value are very similar to the ones of \smp~(Single Tree) whereas for larger $\gr$ values, \smp~converges to the same performance as \smp~(Greedy) since only a single intersection point is considered. Therefore, $\gr$ can be seen as a trade-off between planning faster and achieving lower path costs. \fig \ref{fig:hourglass_samples_cost} compares the path costs for different samples $m$. All methods improve with an increasing amount of samples and \smp~and \smp~(Single Tree) converge to similar cost values.}

\begin{table}[t]
\centering
\renewcommand{\tabcolsep}{4pt}
    \begin{tabular}{|l|c|c|c|}
\hline
                                 & \textbf{Success} & \textbf{Path length} & \textbf{Comp. time [s]} \\
 \hline
 \textbf{3D point w/o obstacles} &                  &                      &                     \\
 PSM${}^*$                       & $10 / 10$        & $\mathbf{14.47\pm 0.04}$      & $10.64\pm 0.16$     \\
 PSM${}^*$ (Single Tree)         & $10 / 10$        & $14.47\pm 0.05$      & $13.72\pm 0.18$     \\
 PSM${}^*$ (Greedy)              & $10 / 10$        & $16.20\pm 0.05$      & $10.36\pm 0.10$     \\
 RRT${}^*$+IK                    & $10 / 10$        & $17.84\pm 2.23$      & $28.35\pm 13.50$    \\
 Random MMP                      & $10 / 10$        & $17.33\pm 1.28$      & $34.75\pm 17.21$    \\
 CBiRRT                          & $10 / 10$        & $14.70\pm 0.71$      & $\mathbf{5.04\pm 0.30}$      \\
 \hline
 \textbf{3D point w/ obstacles}  &                  &                      &                     \\
 PSM${}^*$                       & $10 / 10$        & $15.95\pm 0.13$      & $13.74\pm 0.51$     \\
 PSM${}^*$ (Single Tree)         & $10 / 10$        & $\mathbf{15.87\pm 0.18}$      & $20.42\pm 0.86$     \\
 PSM${}^*$ (Greedy)              & $10 / 10$        & $19.69\pm 0.27$      & $12.89\pm 0.51$      \\
 RRT${}^*$+IK                    & $10 / 10$        & $21.56\pm 3.05$      & $30.21\pm 7.55$     \\
 Random MMP                      & $10 / 10$        & $22.15\pm 2.20$      & $42.09\pm 17.22$     \\
 CBiRRT                          & $10 / 10$        & $16.66\pm 1.34$      & $\mathbf{3.34\pm 0.25}$      \\
 \hline
 \textbf{Robot transport A}      &                  &                      &                     \\
 PSM${}^*$                       & $10 / 10$        & $\mathbf{7.76\pm 0.84}$       & $7.22\pm 0.81$      \\
 PSM${}^*$ (Single Tree)         & $9 / 10$         & $8.53\pm 1.19$       & $11.56\pm 0.41$     \\
 PSM${}^*$ (Greedy)              & $10 / 10$        & $8.18\pm 0.91$       & $\mathbf{7.03\pm 0.09}$      \\
 RRT${}^*$+IK                    & $10 / 10$        & $11.83\pm 3.23$      & $74.65\pm 28.86$    \\
 Random MMP                      & $10 / 10$        & $11.62\pm 2.36$      & $85.45\pm 25.66$    \\
 \hline
 \textbf{Robot transport B}      &                  &                      &                     \\
 PSM${}^*$                       & $10 / 10$        & $\mathbf{14.73\pm 1.27}$      & $\mathbf{14.19\pm 0.76}$     \\
 PSM${}^*$ (Single Tree)         & $0 / 10$         & --                   & --                  \\
 PSM${}^*$ (Greedy)              & $9 / 10$         & $15.41\pm 2.38$      & $14.75\pm 0.53$     \\
 RRT${}^*$+IK                    & $2 / 10$         & $45.14\pm 0.58$      & $89.84\pm 11.00$    \\
 Random MMP                      & $10 / 10$        & $16.96\pm 4.59$      & $163.26\pm 37.80$   \\
 \hline
 \textbf{Robot transport C}      &                  &                      &                     \\
 PSM${}^*$                       & $10 / 10$         & $\mathbf{27.07 \pm 2.58}$      & $\mathbf{275.83 \pm 19.73}$   \\
 PSM${}^*$ (Single Tree)         & $0 / 10$         & --                   & --                  \\
 PSM${}^*$ (Greedy)              & $10 / 10$         & $31.75 \pm 2.51$                   & $311.81\pm 9.79$                  \\
 RRT${}^*$+IK                    & $0 / 10$         & --                   & --                  \\
 Random MMP              & $0 / 10$         & --                   & --                  \\
 \hline
\end{tabular}
\caption{Results of the 3D point and robot transport problems. We report the mean and one unit standard deviation over $10$ runs with different random seeds.}
    \label{tab:comparison_results}
    \vspace{-0.1in}
\end{table}

\subsection{Multi-Robot Object Transport Tasks}
\label{ssec:transport_exp}
In this experiment, we evaluate \smp~on various object transportation tasks involving multiple robots. The overall objective is to transport an object from an initial to a goal location. We consider three variations of this task:
\begin{itemize}[leftmargin=*]
    \item \textbf{Task A}: A single robot arm mounted on a table with $k=6$ degrees of freedom. The task is to transport an object from an initial location on the table to a target location. This task is described by $n=3$ manifolds.
    \item \textbf{Task B}: This task consists of two robot arms and a mobile base consisting of $k=14$ degrees of freedom. The task is defined such that the first robot arm picks the object and places it on the mobile base. Then, the mobile base brings it to the second robot arm that picks it up and places it on the table. This procedure is described with $n=5$ manifolds.
    \item \textbf{Task C}: In this task, four robots are used to transport two objects between two tables. Three arms are mounted on the tables and another arm with a tray is mounted on a mobile base. Besides transporting the objects, the orientation of the two objects needs to be kept upright during the whole motion. This task is described with $n=12$ manifolds and the configuration space has $k=26$ degrees of freedom.
\end{itemize}
The initial states of the three tasks are visualized in \fig \ref{fig:robot_tasks} where the target locations of the objects are shown in green. The geometries of the objects were chosen such that the tasks are intersection point independent (\sect \ref{ssec:ipip}). Three types of constraints are used to describe the tasks. Picking up an object is defined with the constraint
$h_M(q) = x_g - f_{\t{pos}, e}(q)$
where $x_g\in\R^3$ is the location of the object and $f_{\t{pos}, e}(q)$ is the forward kinematics function to a point $e\in\R^3$ on the robot end effector. The handover of an object between two robots is described by
$h_M(q) = f_{\t{pos}, e_1}(q) - f_{\t{pos}, e_2}(q)$
where $f_{\t{pos}, e_1}(q)$ is the forward kinematics function to the end effector of the first robot and $f_{\t{pos}, e_2}(q)$ is that of the second robot.
The orientation constraint is given by an alignment constraint
$h_M(q) = f_{\t{rot}, z}(q) \T e_z - 1$
where $f_{\t{rot}, z}(q)$ is a unit vector attached to the robot end effector that should be aligned with the vector $e_z=(0,0,1)$ to point upwards.
These constraints, \new{when concatenated in different arrangements}, are sufficient to describe the multi-robot transportation tasks. The parameters of the algorithms are summarized in Table \ref{tab:exp_parameter}.

The costs of the algorithms are reported in Table \ref{tab:comparison_results}. 
Note that we only assume access to a forward simulator of the object manipulation operations. CBiRRT could not be applied to this task since it requires a goal configuration at which the backward tree originates. 
This would require a backward simulator and would result in an unfair comparison, since the information that CBiRRT would use during planning would be different from what the other algorithms use.
On Task A, nearly all methods robustly find solutions. However, when task complexity increases, RRT$^*$+IK, \smp~(Single Tree), and Random MMP have difficulties to solve the problem. Only \smp~and \smp~(Greedy) were able to find solutions of task C with the given parameters. The cost and computation time of \smp~is lower compared to \smp~(Greedy). A found solution of \smp~for Task C is visualized in \fig \ref{fig:robot_task_c}.

\subsection{Pick-and-Pour on Panda Robot}
\label{ssec:robot_exp}
\new{In this experiment, we demonstrate \smp~on a real Panda robot arm with seven degrees of freedom. The task description consists of $n=7$ manifolds, which describe the individual grasp, transport, and pouring motions. We assume the position of the objects in the scene are known. The parameters are the same as in the robot transport experiments (Table \ref{tab:exp_parameter}) and the planning time was \unit[131.65]{s}. \fig \ref{fig:panda_pour} shows a found path executed on the real robot.} The resulting motions of the task are shown in the supplementary video.

\begin{figure*}
	\centering
	\subfloat[\label{fig:hourglass_rho_cost}]{\includegraphics[width=.5\textwidth]{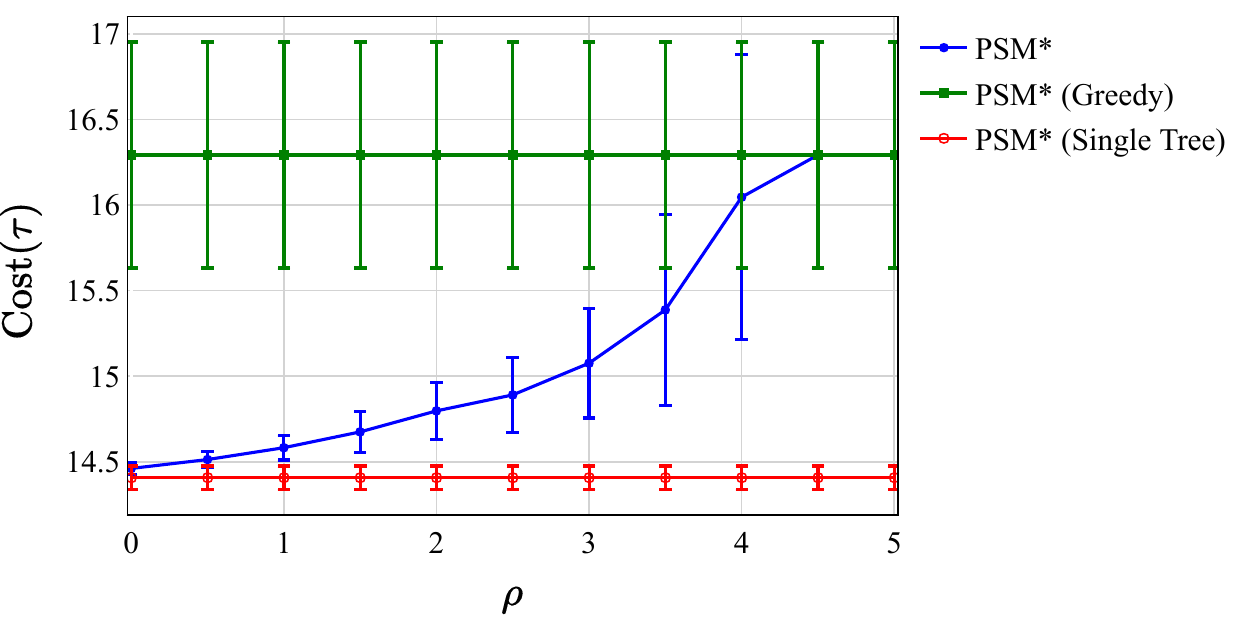}}
	\subfloat[\label{fig:hourglass_samples_cost}]{\includegraphics[width=.5\textwidth]{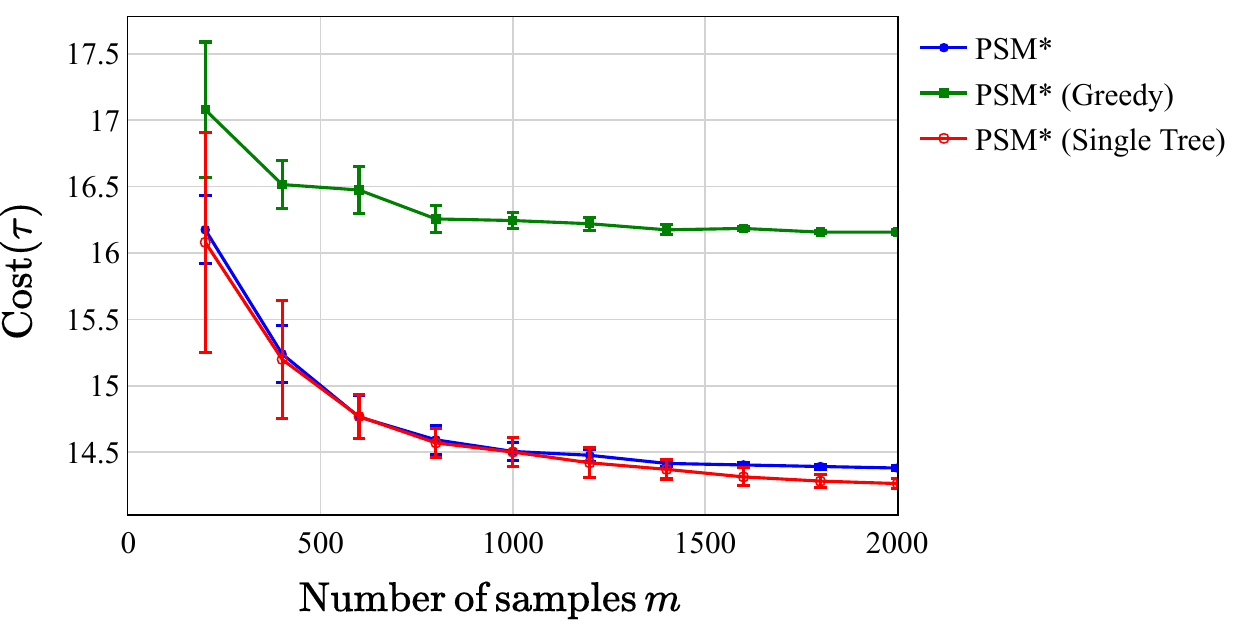}}
	
	\caption{Path costs over variations of the parameter $\gr$ (left) and the number of samples $m$ (right) on the \emph{geometric constraints w/o obstacles} problem (\sect \ref{ssec:geom_exp}). The graphs show the mean and unit standard deviation over $10$ trials. Figure (a) shows the costs increase for higher values of $\gr$, meaning fewer intersection points are considered during planning. 
	In (b), the performance of all methods improves with larger $m$ values where \smp~and \smp~(Single Tree) converge to similar path costs.}
	\label{fig:hourglass_eval}
\end{figure*}

\begin{figure*}
    \centering
    \subfloat[Task A]{\includegraphics[height=0.205\textwidth]{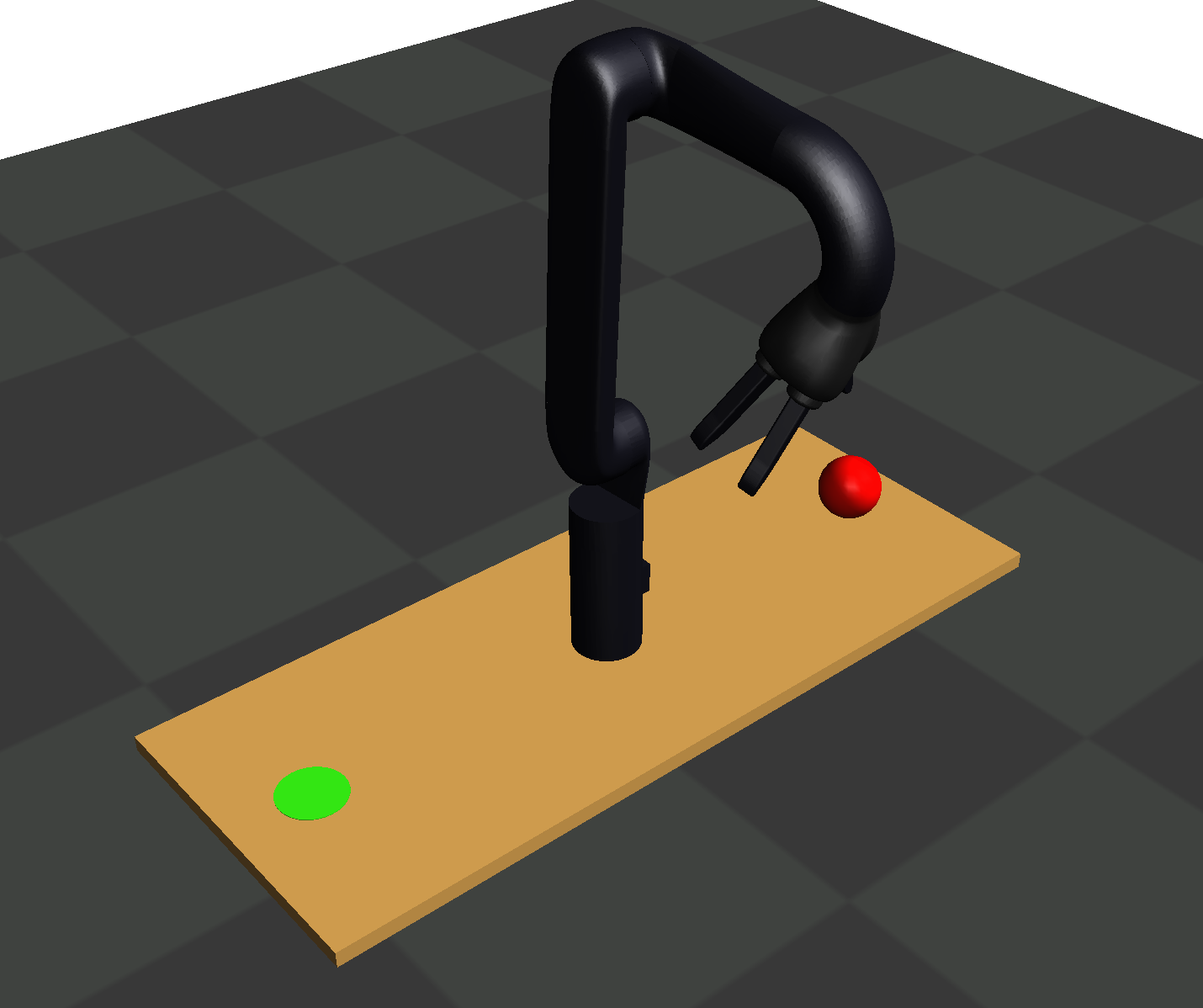}}
    \hspace{0.0in}
    \subfloat[Task B]{\includegraphics[height=0.205\textwidth]{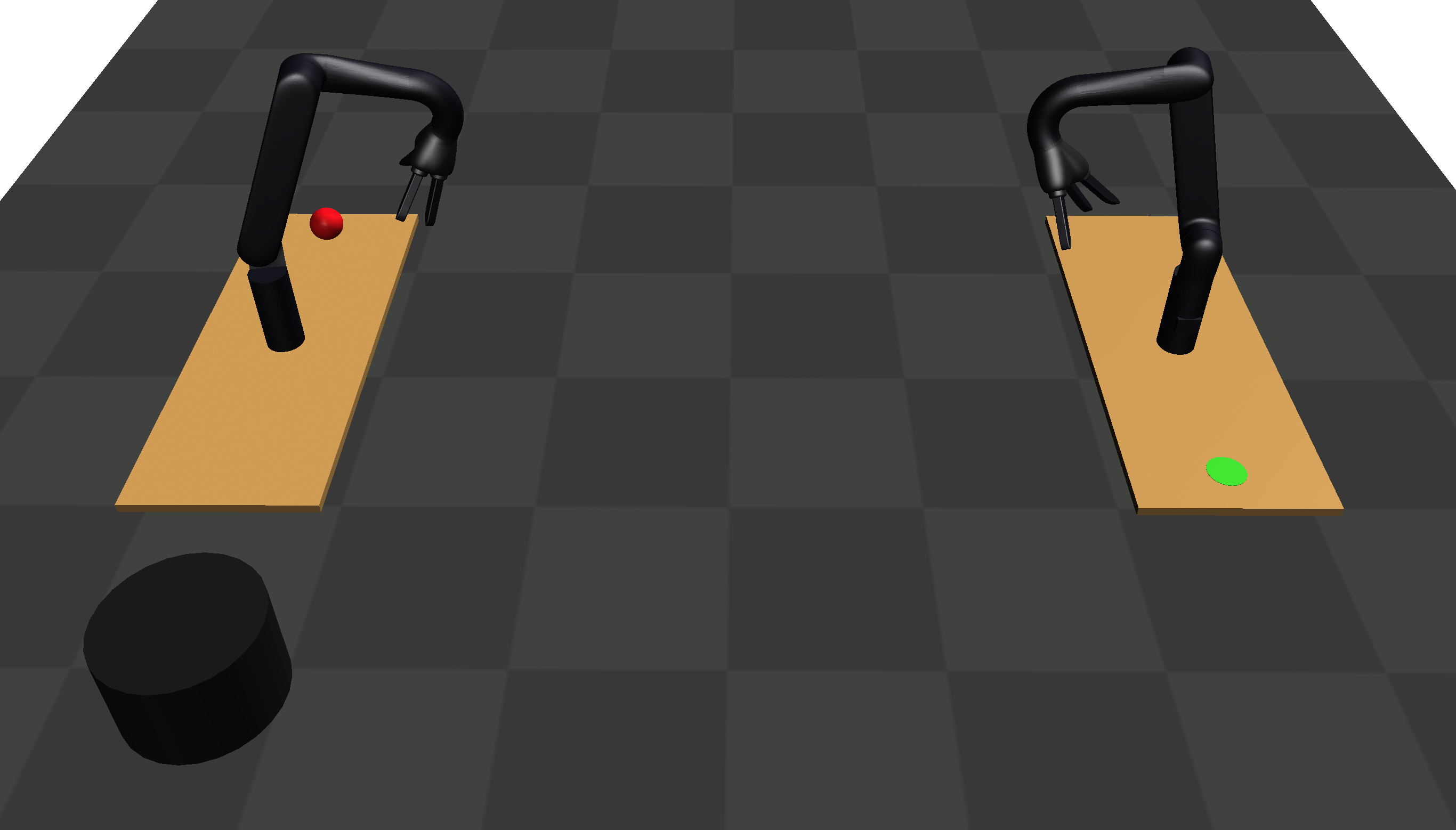}}
    \hspace{0.0in}
    \subfloat[Task C]{\includegraphics[height=0.205\textwidth]{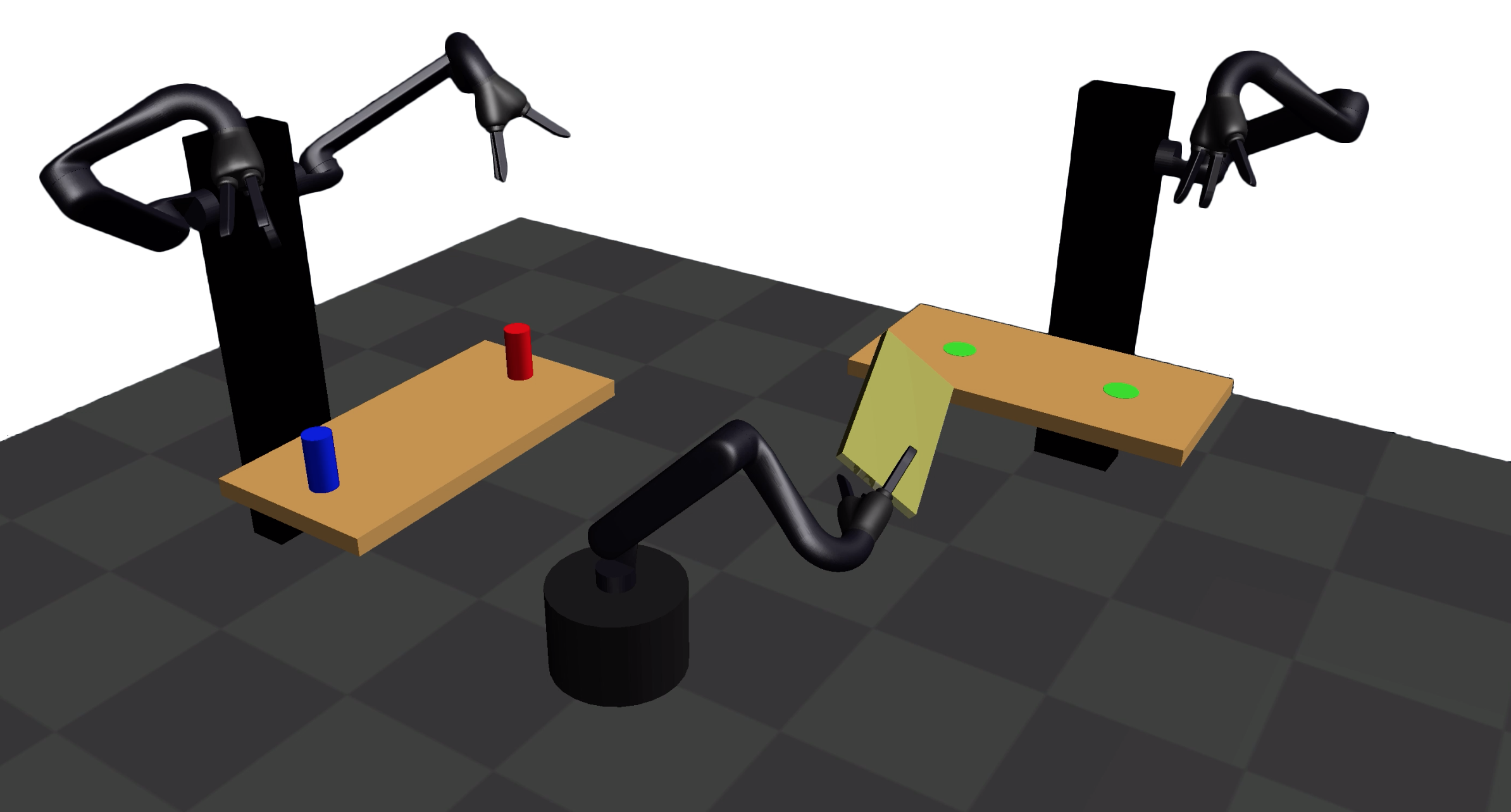}}
    \caption{Start states of the object transport tasks where the goal is to place the red and blue object to the green target locations.}
    \label{fig:robot_tasks}
\end{figure*}

\begin{figure*}
    \centering
    \includegraphics[width=1.0\textwidth]{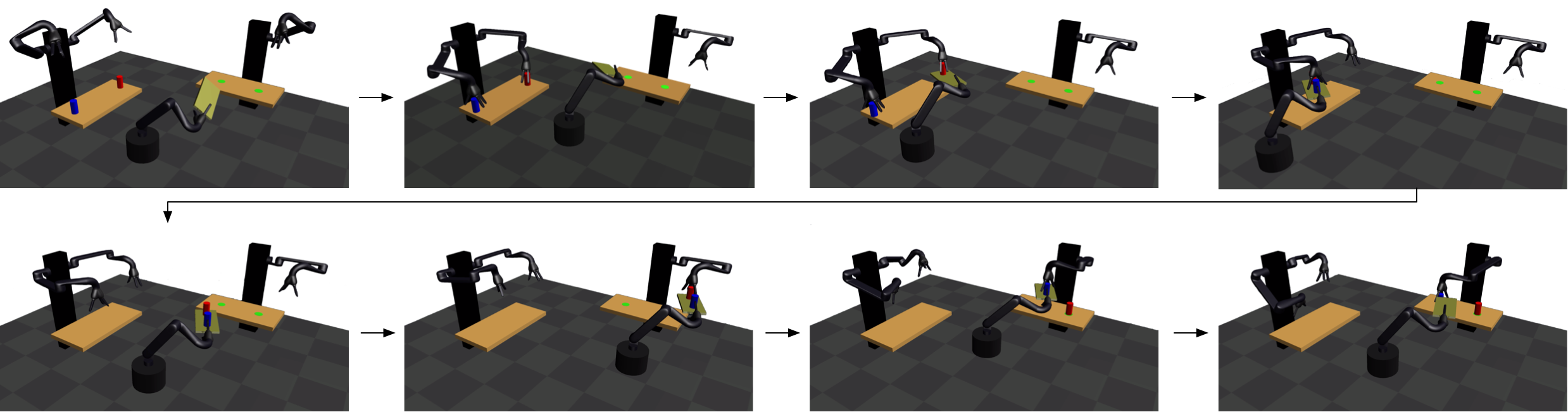}
\caption{\new{Snapshots of the resulting motion that \smp~found for Task C.}}
    \label{fig:robot_task_c}
\end{figure*}
\begin{figure*}
    \centering
    \includegraphics[width=\textwidth]{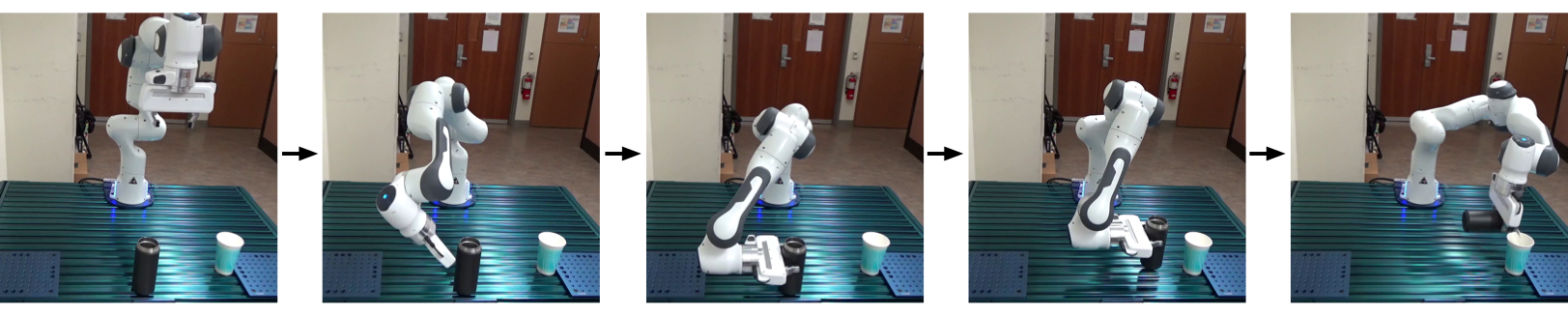}
    \caption{Snapshots of the pouring task motion.}
    \label{fig:panda_pour}
\end{figure*}

\section{Discussion}
\label{sec:conclusion}
We proposed the algorithm \smp~to solve sequential motion planning 
posed as a constrained optimization problem, where the goal is to find a collision-free path that minimizes a cost function and fulfills a given sequence of manifold constraints.
The algorithm is applicable to a certain problem class that is specified by the intersection point independent property, which says that the change in free configuration space is independent of the selected intersection point between two manifolds. 
\smp~uses RRT$^*$'s extend method with a novel steering strategy that is able to discover intersection points between manifolds.
We proved that \smp~is probabilistically complete and asymptotically optimal and demonstrated it on multi-robot transportation tasks.
 
Restricting the problem class to intersection point independent planning problems allowed us to develop efficient solution strategies like growing a single tree over a sequence of manifolds. An interesting question for future research is how to extend the strengths of \smp~to a larger problem class; specifically, how \smp~can be used for problems that do not fulfill the intersection point independent property. For such problems, the choice of intersection points influences the future parts of the planning problem, which results in a more complex problem. An interesting aspect that future work can address is the effect of an intersection point on subsequent manifolds, how solutions for one intersection point can be reused and transferred to other intersection points without replanning from scratch, \new{or investigating a backtracking approach when the tree cannot be further expanded}.
Further, we plan to investigate the reduction of such problems to the simpler intersection point independent problem class, for example, by morphing object geometries into simpler shapes such that the intersection point independent property is satisfied. \smp~could be applied to the reduced problem and provide a good initial guess for solving the original problem.

A future direction we already have begun exploring is the integration of learning techniques into the motion planner \cite{sutanto2020learning}. Here we assumed the user has the knowledge and capability to write the constraint manifolds for the entire task sequence. However, this may not always be the case, or it may be more convenient to give demonstrations of a task such that the constraints can be automatically extracted for future planning. To this end, we are exploring neural network models that are able to learn manifold constraints from robot demonstrations. We train these networks to represent level-set functions of the constraint by aligning subspaces in the network with subspaces of the data. We integrated such learned manifolds together with analytically specified manifolds into a planning problem that was solved with \smp.

\section*{Acknowledgements}
This material is based upon work supported by the National Science Foundation Graduate Research Fellowship Program under Grant No. DGE-1842487. Any opinions, findings, and conclusions or recommendations expressed in this material are those of the author(s) and do not necessarily reflect the views of the National Science Foundation.
This work was supported in part by the Office of Naval Research (ONR) under grant N000141512550.

\bibliographystyle{plainnat}
\bibliography{references}

\clearpage
\appendix
As a reference, we provide in Section \ref{sec:smp_single_tree} a variant of \smp~that grows a single tree over the manifold sequence without splitting it into individual subtrees, which we use as a baseline in the experiments. Afterwards, we summarize in Section \ref{sec:rrt_extend} the extension step of RRT${}^*$ \cite{karaman2011sampling} that is called in the inner loop of the \smp~algorithm. 
In Sections \ref{sec:Probabilistic completeness} and \ref{sec:Asymptotic optimality}, we provide a theoretical analysis regarding the probabilistic completeness and asymptotic optimality of the proposed algorithm.

\subsection{\smp~(Single Tree)}
\label{sec:smp_single_tree}

\alg \ref{alg:smp_single_tree} describes a variant of \smp~that grows a single tree over the manifold sequence without splitting it into individual subtrees. For each node, we store the corresponding current and target manifold as well as the free configuration space, which is extracted by the functions \textit{GetCurrentManifold}, \textit{GetTargetManifold}, and \textit{GetFreeSpace} in lines $5-7$. Given this information, we call the same steering function that \smp~uses and the RRT$^*$ extend routine described in \alg \ref{alg:rrt_extend}. The main difference between \smp~(Single Tree) algorithm and \smp~is that \smp~(Single Tree) grows the tree on all manifolds while \smp~grows a subtree for every manifold in $\mathcal{M}$ in a sequential manner. 

\subsection{RRT$^*$ Extension Step}
\label{sec:rrt_extend}
The RRT$^*$\_EXTEND routine (\alg \ref{alg:rrt_extend}) is a straightforward adaptation of the extension step in RRT$^*$ \cite{karaman2011sampling} to the \smp~problem.
\new{This routine checks the newly projected $q_\t{new}$ configuration for collision with the current free configuration space $C_{\t{free},i}$, performs rewiring steps, and eventually adds it as a node to the tree. 
The rewiring step is equivalent to the one in RRT$^*$ \cite{karaman2011sampling}, which checks if any nearby points can be reached with a shorter distance and updates its parent nodes accordingly.
It uses a distance function $c(q_0, q_1)\in \R_{\geq 0}$ between two nearby configurations and a function \emph{Cost}$(q)$ that stores the path costs from the root of the tree to a node $q$ in order to reconnect a new node $q$ to its neighbor that results in the shortest path from root $q_\t{start}$ to $q$.
Note that $c(q_0, q_1) \triangleq J(\overline{q_0q_1})$, where $\overline{q_0q_1}$ denotes the geodesic \cite{Boothby:107707} joining points $q_0$ and $q_1$ on the current manifold.
}

\begin{algorithm}[t]
\begin{algorithmic}[1]
	\algrenewcommand\algorithmicindent{1.5em}
	\State $V = \{q_{\t{start}}\}$; $E = \emptyset$; $n=\text{len}(\mathcal{M})-1$
        \For{$k = 1$ to $n m$}
			\State $q_\t{rand} \leftarrow \t{Sample}(C)$ \label{alg:st_steer_start}
			\State $q_\t{near} \leftarrow \t{Nearest}(V, q_\t{rand})$
			\State $M_i \leftarrow \t{GetCurrentManifold}(q_\t{near})$
			\State $M_{i+1} \leftarrow \t{GetTargetManifold}(q_\t{near})$
			\State{$q_\t{new} \leftarrow \t{\smp\_STEER} (\alpha, \beta, r, q_\t{near}, q_\t{rand}, M_i, M_{i+1}$)}
	\If{\new{$||h_{M_{i}} (q_\t{new})|| > \ge$}}
	    \State continue
	\EndIf
 		    \State RRT$^*$\_EXTEND$(V_i, E_i, q_\t{near}, q_\t{new},$ \new{$C_{\t{free},i})$}
		\EndFor 
		\State \textbf{return} $\mathrm{OptimalPath}(V, E, q_\t{start}, M_{n+1})$
\end{algorithmic}
\caption{\smp~(Single Tree) $(\mathcal{M}, q_{\t{start}}, \ga, \gb, \ge, r, m)$}
\label{alg:smp_single_tree}
\end{algorithm}
\begin{algorithm}[t]
\begin{algorithmic}[1]
	\algrenewcommand\algorithmicindent{1.0em}
			\If{$\t{CollisionFree}(q_\t{near}, q_\t{new}, $ \new{$C_{\t{free},i})$}}
				\State $Q_\t{near} = \t{Near}\Big(V, q_\t{new}, \min{} \Big\{\gg_\t{RRT*} \left(\tfrac{\log(|V|)}{ |V|}\right)^{1/k}\!, \ga\Big\}\Big)$
				\State $V \leftarrow V \cup \{q_{\t{new}} \}$
				\State $q_\t{min} = q_\t{near}$;  $c_\t{min} = \t{Cost}(q_\t{near}) + c(q_\t{near}, q_\t{new})$
				\For{$q_\t{near}\in Q_\t{near}$}
				\If{$\t{CollisionFree}(q_\t{near}, q_\t{new}, $ \new{$C_{\t{free},i})$} \textbf{and}\\$\quad\qquad\t{Cost}(q_\t{near}) + c(q_\t{near}, q_\t{new}) <c_\t{min}$}
					\State $q_\t{min} = q_\t{near}$; $c_\t{min} = \t{Cost}(q_\t{near}) + c(q_\t{near}, q_\t{new})$
				\EndIf
				\EndFor
				\State $E \leftarrow E \cup \{(q_\t{min}, q_\t{new})\}$
				\For{$q_\t{near}\in Q_\t{near}$}
				\If{$\t{CollisionFree}(q_\t{new}, q_\t{near}, $ \new{$C_{\t{free},i})$} \textbf{and}\\  $\qquad\quad\t{Cost}(q_\t{new}) + c(q_\t{new}, q_\t{near}) <\t{Cost}(q_\t{near})$}
					\State $q_\t{parent} = \t{Parent}(q_\t{near})$
					\State $E \leftarrow E ~\backslash~ \{(q_\t{parent}, q_\t{near})\}$
					\State $E \leftarrow E \cup \{(q_\t{new}, q_\t{near})\}$
				\EndIf
				\EndFor
				\State \textbf{return} True
			\Else
				\State \textbf{return} False
			\EndIf	
\end{algorithmic}
\caption{RRT$^*$\_EXTEND$~(V, E, q_\t{near}, q_\t{new}, $ \new{$C_{\t{free},i})$}}
\label{alg:rrt_extend}
\end{algorithm}
\subsection{Probabilistic Completeness }
\label{sec:Probabilistic completeness}
In this section, we prove the probabilistic completeness of \autoref{alg:smp_single_tree}. Note that, for ease of analysis, the analysis presented in this section and in the subsequent section assumes that $\rho = 0$. 
In the future, we will analyze the properties of our approach when $\rho$ is strictly positive. Moreover, in this analysis we refer to the collision-free region of a manifold when we allude to $M_i$.

\new{In this section, we prove probabilistic completeness and asymptotic analysis for \autoref{alg:smp_single_tree}. We make the claim that they extend to Algorithm \ref{alg:smp} as well based on the following arguments. 
Algorithm \ref{alg:smp} is essentially a sequence of $n$ infinite loops in which, at the termination of the $i^{\text{th}}$ loop, the algorithm will have computed the optimal path from the start location to the intersection of manifold $M_i$ and $M_{i+1}$. On the contrary, \autoref{alg:smp_single_tree} consists of a single infinite loop which upon termination outputs the optimal path from start point through the sequence of manifolds to the goal manifold. 
It is important to note that each loop in the sequence of infinite loops in Algorithm \ref{alg:smp} can be initiated prior to terminating the previous loop. This argument stems from the fact that a search tree in a manifold can be initiated immediately when the nodes in the search tree associated with the previous manifold reach the intersection of the manifold and the previous manifold. Therefore, each infinite loop in Algorithm \ref{alg:smp} can be viewed as being executed in parallel and can be assumed to terminate simultaneously.  
Additionally, it is straightforward to understand that at the end of each infinite loop, Algorithm \ref{alg:smp} solves a subproblem of the problem that \autoref{alg:smp_single_tree} solves. 
Therefore, by the principle of optimality \cite[Chapter 5]{CoV_Liberzon}, both \autoref{alg:smp_single_tree} and Algorithm \ref{alg:smp} solve the same problem and are identical for our analysis. 
}

\begin{definition}
A collision-free path is said to have \emph{strong} $\delta$-\emph{clearance}  if the path lies entirely inside the $\delta$-interior of $\cup\mathcal{M}$, where $\cup \mathcal{M} \triangleq \cup_{i=1}^{n+1} M_i$ \cite{karaman2011sampling}. 
\end{definition}

We start by assuming that there exists a path $\Hat{\v\tau}$ with \emph{strong} $\delta$-\emph{clearance} connecting the goal manifold $M_{n+1}$ with the start configuration $q_\t{start}$ embedded on the sequence of manifolds under consideration. Let $L$ be the total length of the path, computed based on the pullback metric \cite{lee2006riemannian} of the manifolds due to their embedding in $\R^{k}$. Let $\xi>0$ be the minimum over the reach \cite{aamari2019estimating} of all manifolds in the sequence and the manifolds resulting from the pairwise intersection of adjacent manifolds in the sequence. Informally, the reach of a manifold is the size of an envelope around the manifold such that any point within the envelope and the manifold has an unique projection onto the manifold.  For the analysis presented here, we pick the steering parameter $\alpha$ such that $\xi \geq \alpha$. We use the notation $\textit{Tube}(M_i,\xi)$ to denote the set $\{x\in \R^k ~|~ d(x,M_i)<\xi\}$, where 
\begin{align}
    d(x,M_i) = \inf \{\|x-y\|_{\R^k}|\ y \in M_i\}
\end{align}
{is the minimum distance of the point $x$ to the manifold}. Now, if we define 
\begin{align}
\zeta_i = \sup\limits_{q \in \textit{Tube}(M_i,\xi)} \|h_{M_i}(q)\|,
\end{align}
then for the sake of analysis we assume that $r = \max\ \{\zeta_1,\cdots, \zeta_{n+1}\}$. If $\nu = \min\ (\delta, \alpha)$, then we define a sequence of points
\begin{align}
    \{&[q^1_0=q_\t{start},q^1_1, \cdots, q^1_{m_1}],[q^2_0,q^2_1, \cdots, q^2_{m_2}],\cdots,\nonumber\\  &[q^n_0,q^n_1, \cdots, q^n_{m_1}], [q^{n+1}_0]\}
\end{align} 
on $\Hat{\v\tau}$, such that $[q^i_0,q^i_1, \cdots, q^i_{m_i}] \in M_i$ and $\sum_{i=1}^{n} m_i = m$, where $m=\frac{5L}{\nu (n+1)}$ is the total number of points in the path. Without loss of generality, we assume that for every $M_i$ with $1<i<n+1$ there exists a non-negative integer $j<m_i$ such that  $q^i_0,q^i_1, \cdots, q^i_j \in M_{i-1} \cap M_{i}$ and $q^i_{j+1},q^i_{j+2}, \cdots, q^i_{m_i} \in M_{i} \setminus M_{i+1}$. In other words, there exist some points at the beginning of $[q^i_0,q^i_1, \cdots, q^i_{m_i}]$ that belong to $M_{i-1} \cap M_{i}$ and the rest of the points on the manifold belong exclusively to {$M_i$}.  For ease of analysis, the sequence of points on $\Hat{\v\tau}$ is chosen such that 
\begin{align}
      \|q^i_j - q^i_{j+1}\|_{\R^k} \leq \|q^i_j - q^i_{j+1}\|_{M_i} \leq \nu/5
\end{align}
\new{where   $\|\cdot\|_{\R^k}$ and $\|\cdot\|_{M_i}$ are the distances between the points according to the metrics on the   ambient space and manifold respectively.}
We use $B(q^i_j,\nu) \subset \R^k$ to denote a ball of radius $\nu$ around $q^i_j$ under the standard Euclidean norm on $\R^k$. We denote the tree that is grown with RRT$^*$ as $T=(V, E)$.

We prove the probabilistic completeness of our strategy in two parts. The first part proves the probabilistic completeness of RRT$^*$ on a single manifold. In the second part, we prove that with probability one, the tree $T$ grown on a manifold can be expanded onto the next manifold as the number of samples tends to infinity. For the first part, as suggested in \cite[Section 5.3]{kingston2019ijrr}, the Lemma 1 in \cite{Kleinbort2019} can be shown to hold for the single manifold case and probabilistic completeness of RRT/RRT$^*$ on a manifold can be easily proven using \cite[Theorem 1]{Kleinbort2019}. We now focus on proving the second part that shows the probabilistic completeness of our strategy. 
We start by proving \autoref{lemma} which  enables us to prove that a tree grown with RRT$^*$ on a manifold can be extended to the next manifold. 

\begin{lemma}
\label{lemma}
Suppose that ${T}$ has reached $M_i$ and contains a vertex $\Tilde{q}^i_{m_i}$ such that $\Tilde{q}^i_{m_i} \in B(q^i_{m_i},\nu / 5)$. If a random sample $q_\t{rand}^{i+1}$ is drawn such that $q_\t{rand}^{i+1} \in B(q^{i+1}_{0},\nu / 5)$, then the straight path between $\textit{Project}(q_\t{rand}^{i+1}, M_{i} \cap M_{i+1})$ and the nearest neighbor $q_\t{near}$ of $q_\t{rand}^{i+1}$ in $T$ lies entirely in $C_{\t{free},i+1}$.
\end{lemma}
\begin{proof}
By definition $\|q_\t{near}-q_\t{rand}^{i+1}\| \leq \|\Tilde{q}^i_{m_i}-q_\t{rand}^{i+1}\|$, then using the triangle inequality and some algebraic manipulation similar to that used in the proof of \cite[Lemma 1]{Kleinbort2019}, we can show that
\begin{align}
    \|q_\t{near}-q^i_{m_i}\| ~\leq~ &\|\Tilde{q}^i_{m_i}-q^i_{m_i} \| + 2 \|x^{i+1}_{0} - q_\t{rand}^{i+1} \| \nonumber\\ &+2\|q^i_{m_i}-q^{i+1}_{0}\|
\end{align}
which leads to $\|q_\t{near}-q^i_{m_i}\| \leq 5 \frac{\nu}{5}\leq \nu$ and \mbox{$q_\t{near} \in B(q^i_{m_i},\nu)$.} Again, by the triangle inequality, we can show that \mbox{$\|q_\t{near}-q_\t{rand}^{i+1}\| \leq \nu$.} As the sample $q_\t{rand}^{i+1}$ is taken from within the reach of $M_i$ there exists a unique nearest point of $q_\t{rand}^{i+1}$ on $M_i$ \cite{aamari2019estimating}. In other words, the operation $\textit{Project}(q_\t{rand}^{i+1}, M_i \cap M_{i+1})$ is well-defined. Therefore, as
\begin{align}
    \|q_\t{near}-\textit{Project}(q_\t{rand}^{i+1}, M_i \cap M_{i+1})\| \leq \|q_\t{near}-q_\t{rand}^{i+1}\| \leq \nu,
\end{align} the straight path between $\textit{Project}(q_\t{rand}^{i+1},  M_i \cap M_{i+1})$ and   $q_\t{near}$ lies entirely in $C_{\t{free},i+1}$.
\end{proof}

Note that the above lemma is an extension of \cite[Lemma 1]{Kleinbort2019}. The next theorem will prove that with probability one \smp~will yield a path as the number of samples goes to infinity. Since we are only concerned about the transition of $T$ from one manifold to the next, we focus on the iterations in \smp~after $T$ reaches a neighborhood of $q^i_{m_i} \in M_i$. We refer to such an iteration as a \textit{boundary iteration}. 

\begin{theorem}
The probability that \smp~fails to reach the final manifold $M_{n+1}$ from an initial configuration after $t$ boundary iterations is bounded from above by
$a \exp{(-b t)}$, for some positive real numbers $a$ and $b$.
\end{theorem}
\begin{proof}
Let $B(q^i_{m_i},\nu / 5)$ contain a vertex of $T$. Let $p$ be the probability that in a boundary iteration a vertex contained in $B(q^{i+1}_{0},\nu / 5)$ is added to $T$. From \autoref{lemma}, if we obtain a sample $q_\t{rand}^{i+1} \in B(q^{i+1}_{0},\nu / 5)$, then $T$ can reach $\textit{Project}(q_\t{rand}^{i+1}, M_i \cap M_{i+1})$. The value $p$ can be computed as a product of the probabilities of two events: 1) a sample is drawn from $B(q^{i+1}_{0},\nu / 5)$, and 2) $T$ is extended to include $\textit{Project}(q_\t{rand}^{i+1}, M_i \cap M_{i+1})$. 
The probability that a sample is drawn from $B(q^{i+1}_{0},\nu / 5)$ is given by $|B(q^{i+1}_{0},\nu / 5)|/|C|$, where $|B(q^{i+1}_{0},\nu / 5)|$ and $|C|$ are the volumes of $B(q^{i+1}_{0},\nu / 5)$ and $C$ respectively. From the proof of \autoref{lemma}, we infer that the line joining $q_\t{near}$ and $q_\t{rand}^{i+1}$ is collision-free. 
Thus $T$ will be augmented with a new vertex contained in $M_i \cap M_{i+1}$ if line 2 and 8 in Algorithm \ref{alg:steer_project} are executed. The probability of execution of line 2 and 8 is $(1-\beta) \frac{r-\|h_{M_{i+1}}(q_\t{new})\|}{r}$, which results in the joint probability
\begin{align}
    p= \frac{|B(x^{i+1}_{0},\nu / 5)|}{|C|}(1-\beta)\frac{r-\|h_{M_{i+1}}(q_\t{new})\|}{r}.
\end{align}
Further,  $\|h_{M_{i+1}}(q_\t{new})\| \approx 0$ as $q_\t{new}$ is very close to $M_{i+1}$ and $ |B(x^{i+1}_{0},\nu / 5)| \ll |C|$, thus $\beta$ can be picked such that $p < 0.5$. For \smp~to reach $M_{i+1}$ from the initial point, the boundary iteration should successfully extend $T$ for at least $n$ times (there are $n$ intersections between the $n+1$ sequential manifolds). The process can be viewed as $t > n$ Bernoulli trials with success probability $p$. Let $\Pi_t$ denote the number of successes in $t$ trials,  then 
\begin{align}
    \mathbb{P}\left[\Pi_t < n\right] = \sum_{i=0}^{n-1}{t\choose i}p^i(1-p)^{t-i}\comma
\end{align}
where $\mathbb{P}[\cdot]$ denotes the probability of occurrence of an event. By using the fact that $n \ll t$, this can be upper bounded as \begin{align}
    \mathbb{P}\left[\Pi_t < n\right] &\leq \sum_{i=0}^{n-1}{t\choose n-1}p^i(1-p)^{t-i},
\end{align}
as $p < 0.5$,
\begin{align}
    \mathbb{P}\left[\Pi_t < n\right] \leq {t\choose n-1} \sum_{i=0}^{n-1} (1-p)^{t}~.
\end{align}
Applying $(1-p)\leq \exp{(-p)}$ yields
\begin{align}
    \mathbb{P}\left[\Pi_t < n\right] \leq n {t\choose n-1}  (\exp{(-pt)}).
\end{align}
Through further algebraic simplifications, we can show that 
\begin{align}
    \mathbb{P}\left[\Pi_t < n\right] \leq \frac{n}{(n-1)!}t^n\exp{(-p t)}~.
\end{align}
\end{proof}
As the failure probability of \smp~exponentially goes to zero as $t \to \infty$, \smp~is probabilistically complete.

\subsection{Asymptotic Optimality}
\label{sec:Asymptotic optimality}
For ease of reference, we begin by giving some definitions and stating some lemmas initially introduced in \cite{karaman2011sampling}, which are required for proving the asymptotic optimality of \smp.

\begin{definition}
A path $\tau_1$ is said to be homotopic to $\tau_2$ if there exists a continuous function $H : [0,1] \times [0,1] \rightarrow \cup \mathcal{M}$, called the \emph{homotopy} \cite{hatcher2002algebraic}, such that $H(t,0) = \tau_1(t)$, $H(t,1) = \tau_2(t)$, and $H(\cdot, \alpha)$ is a collision-free path for all $\alpha \in [0,1]$. 
\end{definition}

\begin{definition}
A collision-free path $\tau : [0,1] \rightarrow \cup\mathcal{M}$ is said to have \emph{weak} $\delta$-\emph{clearance} \cite{karaman2011sampling} if there exists a path $\tau'$ that has strong $\delta$-clearance and there exists a homotopy $H : [0,1] \times [0,1] \rightarrow \cup\mathcal{M}$ with $H(t,0)=\tau(t)$, $H(t,1)=\tau'(t)$, and for all $\alpha \in (0,1]$ there exists $\delta_{\alpha}>0$ such that $H(t, \alpha)$ has strong $\delta_{\alpha}$-clearance.
\end{definition}

\begin{lemma}
\label{lemma: strong to weak}
\cite[Lemma 50]{karaman2011sampling}
Let $\tau^*$ be a path with weak $\delta$-clearance. Let $\{\delta_n\}_{n\in \mathbb{N}}$ be a sequence of real numbers such that $\lim{n \rightarrow \infty} \delta_n=0$ and $0\leq\delta_n\leq\delta$ for all $n \in \mathbb{N}$. Then, there exists a sequence $\{\tau_n\}_{n\in \mathbb{N}}$ of paths such that $\lim{n \rightarrow \infty} \tau_n=\tau^*$ and $\tau_n$ has strong $\delta_n$-clearance for all $n \in \mathbb{N}$.
\end{lemma}
The above lemma establishes the relationship between the weak and strong $\delta$-clearance paths.

If the configuration space admits a vector space structure, then it can be shown that the space of paths on the configuration space above also admits a vector space structure. Moreover, the space of path becomes a normed space if it is endowed with the \emph{bounded variation} norm \cite{Stein1385521}
\begin{align}
    \|\tau\|_{\text{BV}} \triangleq \int_0^1 |\tau(t)| dt~+~ \text{TV}(\tau).
\end{align}
$\text{TV}(\tau)$ denotes the \emph{total variation} norm \cite{Stein1385521} defined as 
\begin{align}
    \text{TV}(\tau)=\sup_{\{n\in \mathbb{N}, 0=t_1<t_2<...<t_n=1\}} \sum_{i=1}^{n} |\tau(t_i)-\tau(t_{i-1})|.
\end{align}
Using the norm in the space of paths, the distance between the paths $\tau_1$ and $\tau_2$ can defined as 
\begin{align}
\label{eqn: path distance}
    \|\tau_1 - \tau_2\|_{\text{BV}} = \int_0^1 |\tau_1(t)-\tau_2(t)| dt~+~ \text{TV}(\tau_1-\tau_2).
\end{align}
The normed vector space of paths enables us to mathematically formulate the notion of the convergence of a sequence of paths to a path. Formally, given a sequence of paths $\{\tau_n\}, n \in \mathbb{N}$, the sequence converges to a path $\Bar{\tau}$, denoted as $\lim{n \rightarrow \infty} \tau_n = \Bar{\tau}$, if $\lim{n \rightarrow \infty} \|\tau_n - \Bar{\tau}\|_{\text{BV}}=0$.

Let $\mathcal{P}$ denote the set of weak $\delta$-clearance paths which satisfies the constraints in Equation 1. Let $\tau^* \in \mathcal{P}$ be a path with the minimal cost. Due to the continuity of the cost function, any sequence of paths $\{\tau_n \in \mathcal{P}\}, n \in \mathbb{N}$ such that $\lim{n \rightarrow \infty} \tau_n = \tau^*$ also satisfies $\lim{n \rightarrow \infty} J(\tau_n) = J(\tau^*)$. For brevity, we identify $J(\tau^*)$ with $J^*$ and $J^{\text{\smp}}_n$ denotes the random variable modeling the cost of the minimum-cost solution returned by \smp~after $n$ iterations. The \smp~algorithm is asymptotically optimal if
\begin{align}
    \label{eqn: asymp opt def}
    \mathbb{P}\left[\lim{n\rightarrow  \infty} J^{\text{\smp}}_n = J^* \right] = 1.
\end{align}
A weaker condition  than \autoref{eqn: asymp opt def} is 
\begin{align}
    \mathbb{P}\left[\limsup_{n\rightarrow  \infty} J^{\text{\smp}}_n = J^* \right] = 1.
\end{align} 
Note that from \cite[Lemma 25]{karaman2011sampling}, we infer that the probability that $\limsup_{n\rightarrow  \infty} J^{\text{\smp}}_n = J^*$ is either zero or one. Under the assumption that the set of points traversed by an optimal path has measure zero, \cite[Lemma 28]{karaman2011sampling} proves that the probability that \smp~returns a tree containing an optimal path in finite number of iterations is zero. 

Since \smp~is based on RRT$^*$, we focus on how our technique affects the proofs of asymptotic optimality for RRT$^*$. Also, the work in \cite{kingston2018sampling} has shown that RRT$^*$ is optimal when applied on a single manifold. Furthermore, it is shown in \cite{kingston2018sampling} that the steering parameter $\gamma$ in the single manifold case can be bounded from below by $\left(2 \left(1 + \frac{1}{k}\right)\frac{\mu \left( \cup_{q \in C_\t{free}} D_q\right)}{\zeta_{M}(1)}\right)^{\frac{1}{k}}$, where $D_q$ is the set of points in $\R^k$ which are projected on $q$ and $\zeta_{M}(1)$ is the set points in $\R^k$ which are projected onto a unit open ball contained in the manifold $M$. \new{Also,  $\mu(\cdot)$ denotes a \emph{measure} \cite{Stein1385521} on the \emph{measurable space}  \cite{Stein1385521} $\R^k$. Intuitively, a measure maps the subset of a measurable space to its volume.} In this section, we show that under the assumption $\rho = 0 $ in Algorithm \ref{alg:smp}, with probability one \smp~eventually returns the optimal path. 

Let $\{Q_1, Q_2,..., Q_n\}$ be a set of independent uniformly distributed points drawn from $C_\t{free}$ and let $\{I_1, I_2,..., I_n\}$ be their associated labels that outlines the order of the points  with support $[0,1]$. In other words, a point $Q_i$ is assumed to be drawn after another point $Q_{\Bar{i}}$ if $I_i < I_{\Bar{i}}$. 
Let $\{\Hat{Q}_1, \Hat{Q}_2,..., \Hat{Q}_n\}$ be the set of points resulting from projecting the point set onto the manifolds as delineated in lines 7--10 of Algorithm \ref{alg:steer_project}. Similar to \cite{karaman2011sampling}, we consider the graph formed by adding an edge $(\Hat{Q}_i, \Hat{Q}_{\Bar{i}})$, whenever (i) $I_i < I_{\Bar{i}}$ and \mbox{(ii) $\|\Hat{Q}_i-\Hat{Q}_{\Bar{i}}\| \leq r_n = \gamma_{\t{RRT}^*}(\frac{\log(|V_n|)}{|V_n|})^{\frac{1}{k}}$,} where $V_n$ is the vertex set of the graph. 
Let this graph be denoted by $\mathcal{G}_n=(V_n, E_n)$, \new{where $V_n$ and $E_n$ are the set of vertices and edges of $\mathcal{G}_n$ respectively}. With slight abuse of notation, $J_n^{\text{\smp}}(\Hat{Q}_i)$ denotes the cost of the best path starting from $q_\t{start}$ to the vertex $\Hat{Q}_i$ in the graph $\mathcal{G}$. 
Consider the tree $\Bar{\mathcal{G}}_n$ which is a subgraph  of $\mathcal{G}_n$ where the cost of reaching the vertex $\Hat{Q}_i$ equals $J_n^{\text{\smp}}(\Hat{Q}_i)$. 
Since \smp~uses RRT$^*$ for graph construction, it is easy to see that the tree returned by \smp~at the $n$-th iteration is equivalent to $\Bar{\mathcal{G}}_n$. 
Therefore, if $\limsup_{n \rightarrow \infty} J_n^{\text{\smp}}(M_{n+1})$ converges to $J^*$ with probability one with respect to $\mathcal{G}_n$, then it implies that with probability one \smp~will eventually  return a tree that contains the optimal path connecting $q_\t{start}$ and the goal manifold $M_{n+1}$. 
Hence, our next step is focused on showing that the optimal path in $\mathcal{G}_n$ converges to $\tau^*$.

According to \autoref{lemma: strong to weak}, there exists a sequence of strong $\delta-$clearance paths $\{\tau_m\}_{m \in \mathbb{N}}$ that converges to an optimal path $\tau^*$. Let $B_m\triangleq\{B_{m,1}, B_{m,2},...,B_{m,p}\}$ be a set of open balls of radius $r_n$ whose centers lie on the path $\tau_m$ such that adjacent balls are placed $2r_n$ distance apart. The number of balls $p$ is assumed to be large enough to cover $\tau_m$, i.e. $\tau_m \setminus \left(\cup_{i =1 }^{p} B_{m,i} \cap \tau_m\right)$ is a set of measure zero. Moreover, we denote $\Tilde{B}_{m,i}$ as the region obtained as the intersection of the open ball $B_{m,i}$ with the manifold containing its center. Let $\Theta_{m,i}$ denote the event that there exists vertices $\Hat{Q}_i$ and $\Hat{Q}_{\Bar{i}}$ such that $\Hat{Q}_i \in \Tilde{B}_{m,i}$, $\Hat{Q}_{\Bar{i}} \in \Tilde{B}_{m,i+1}$ and $I_{i} \geq I_{\Bar{i}}$. Recall that, $I_{i}$ and $ I_{\Bar{i}}$ are the labels associated with projected points $\Hat{Q}_i$ and $\Hat{Q}_{\Bar{i}}$ respectively. Also note that the edge $(\Hat{Q}_i, \Hat{Q}_{\Bar{i}})$ is included in $\mathcal{G}_n$. 
$D_q$ denotes the set of points that can be projected on the point $q \in  \cup \mathcal{M}$. Additionally, $\zeta(1)$ is defined as 
\begin{align}
\zeta(1) \triangleq \max{M \in \mathcal{M}} \min{q \in M} \mu \left(\cup_{q \in \Tilde{B}_{M}(q',1)} D_q\right),
\end{align}
where $\Tilde{B}_{M_i}(q',1)$ is formed by the intersection of a open unit ball centered at point $q'\in M$ with $M$.
If $\Theta_m = \cap_{i=1}^{p} \Theta_{m,i}$, then the following lemma proves that with probability one, the event $\Theta_{m,i}$ for all $i \in \{1,2,...,p\}$ occurs for large $m$.
\begin{lemma}
If 
\begin{align}
    \gamma > \left(2 \left(1 + \frac{1}{k}\right) \left(\frac{\mu( \cup_{q \in \cup \mathcal{M}} D_q)}{\zeta(1)}\right)\right)^{\frac{1}{k}},
\end{align}
then $\Theta_m$ occurs for all large $m$, with probability one, i.e., $\mathbb{P}(\liminf_{n \rightarrow \infty} \Theta_m) = 1$.
\end{lemma}

The proof of the above lemma follows from the proof of \cite[Lemma 71]{karaman2011sampling} if we replace $\mu(C_\t{free})$ with $\mu( \cup_{q \in \cup \mathcal{M}} D_q)$ and infer that the probability of finding a vertex of the graph in $\Tilde{B}_{m,i}$ is $\frac{\zeta(1)}{\mu( \cup_{q \in \cup \mathcal{M}} D_q)}$.
If $\mathcal{L}_n$ denotes the set of all paths that satisfy the constraints in Equation 1 and are contained in the tree returned by \smp~after $n$ iterations such that $\Bar{\tau}^{\text{\smp}}_n \triangleq \min{\tau^{\text{\smp}} \in \mathcal{L}_n} \|\tau^{\text{\smp}} - \tau_m  \|_{\text{BV}}$, then the following lemma can be proven.

\begin{lemma}
\cite[Lemma 72]{karaman2011sampling} The random variable \mbox{$\|\Bar{\tau}^{\text{\smp}}_n - \tau_m\|_{\text{BV}}$} converges to zero with probability one: 
\begin{align}
\label{eqn:grph pth 2 seq pth}
    \mathbb{P}\left[\lim{n \rightarrow \infty} \|\Bar{\tau}^{\text{\smp}}_n - \tau_m\|_{\text{BV}} = 0\right] =1.
\end{align}
\end{lemma}

Recall that by construction $\lim{m \rightarrow \infty} \tau_m = \tau^*$. Expressing \autoref{eqn:grph pth 2 seq pth} as
\begin{align}
    \mathbb{P}\left[\lim{n \rightarrow \infty} \|\Bar{\tau}^{\text{\smp}}_n - \tau^* -(\tau_m -\tau^*)\|_{\text{BV}} = 0\right] =1
\end{align}
and applying the triangle inequality yields
\begin{align*}
    \|\Bar{\tau}^{\text{\smp}}_n\!-\!\tau^* \!-\!(\tau_m \!-\!\tau^*)\|_{\text{BV}} \geq \|\Bar{\tau}^{\text{\smp}}_n - \tau^*\|_{\text{BV}} - \|\tau_m -\tau^*\|_{\text{BV}}.
\end{align*}
From \autoref{eqn:grph pth 2 seq pth} and since $\mathbb{P}\left[\lim{m \rightarrow \infty} \|\tau_m -\tau^*\|_{\text{BV}} = 0 \right]=1$ yields the following result:
\begin{align}
    \mathbb{P}\left[\lim{n \rightarrow \infty} \|\Bar{\tau}^{\text{\smp}}_n - \tau^*\|_{\text{BV}} = 0 \right]=1.
\end{align}
From the continuity of the cost function and due to the fact that $J^{\text{\smp}}_{i+1} \leq J^{\text{\smp}}_{i}$, $i \in \mathbb{N}$ and $J^{\text{\smp}}_{i} \geq J^*$ we obtain the required result (\autoref{eqn: asymp opt def}).

\end{document}